\documentclass{article}

\usepackage[preprint]{styfile}
\usepackage{microtype}
\usepackage{graphicx}
\usepackage{subcaption}
\usepackage{booktabs} 
\usepackage{wrapfig}
\usepackage{float}
\usepackage{hyperref}
\usepackage{comment}
\usepackage{enumitem}
\usepackage{booktabs}
\usepackage{multirow}
\usepackage{makecell}
\usepackage{booktabs}
\usepackage[table]{xcolor}
\usepackage{pifont}
\usepackage{algorithm}
\usepackage{algpseudocode}
\usepackage{amsmath}
\newcommand{\cmark}{\ding{51}}
\newcommand{\xmark}{\ding{55}}

\usepackage[utf8]{inputenc} 
\usepackage[T1]{fontenc}    
\usepackage{hyperref}       
\usepackage{url}            
\usepackage{booktabs}       
\usepackage{amsfonts}       
\usepackage{nicefrac}       
\usepackage{microtype}      
\usepackage{xcolor}         
\usepackage{graphicx}
\usepackage{amsmath}
\usepackage{amssymb}
\usepackage{mathtools}
\usepackage{amsthm}

\usepackage{multirow}
\usepackage{multicol}

\usepackage[capitalize,noabbrev]{cleveref}

\theoremstyle{plain}
\newtheorem{theorem}{Theorem}[section]
\newtheorem{proposition}[theorem]{Proposition}
\newtheorem{lemma}[theorem]{Lemma}

\theoremstyle{definition}

\theoremstyle{remark}

\def \bE {\mathbb{E}}

\def \bN {\mathbb{N}}

\title{Couple to Control: Joint Initial Noise Design in Diffusion Models}

\author{%
  Jing Jia\\
  Department of Computer Science\\
  Rutgers University\\
  Piscataway, NJ 08854 \\
  \texttt{jing.jia@rutgers.edu} \\
  \And 
  Liyue Shen \\
  Department of EECS\\ University of Michigan\\
  Ann Arbor, MI 48109\\
  \texttt{liyues@umich.edu} \\
  \And
  Guanyang Wang\\
  Department of Statistics\\
  Rutgers University\\
  Piscataway, NJ 08854 \\
  \texttt{guanyang.wang@rutgers.edu} \\
}

\begin{document}

\maketitle

\begin{abstract}
Diffusion models typically generate image batches from independent Gaussian initial noises. We argue that this independence assumption is only one choice within a broader class of valid joint noise designs. Instead, one can specify a coupling of the initial noises: each noise remains marginally standard Gaussian, so the pretrained diffusion model receives the same single-sample input distribution, while the dependence across samples is chosen by design. This reframes initial-noise control from selecting or optimizing individual seeds to designing the dependence structure of a multi-sample gallery. This view gives a general framework for initial-noise design, covering several existing methods as special cases and leading naturally to new coupled-noise constructions. Coupled noise can improve generation on its own without adding sampling cost, and it is flexible enough to serve as a structured initialization for optimization-based pipelines when additional computation is available. Empirically, repulsive Gaussian coupling improves gallery diversity on SD1.5, SDXL, and SD3 while largely preserving prompt alignment and image quality. It matches or outperforms recent test-time noise-optimization baselines on several diversity metrics at the same sampling cost as independent generation. Subspace couplings also support fixed-object background generation, producing diverse, natural backgrounds compared with specialized inpainting baselines, with a tunable trade-off in foreground fidelity. 

.
\end{abstract}

\section{Introduction}\label{sec:intro}

Modern diffusion models have achieved remarkable photorealism, but their practical value increasingly depends on two goals: \textit{quality}, producing realistic and prompt-faithful images, and \textit{control}, following user-specified structure across generated outputs.

The quality objective has been extensively studied and engineered. Guidance mechanisms such as classifier guidance \citep{dhariwal2021diffusion}, classifier-free guidance \citep{ho2022classifier}, and CLIP guidance \citep{radford2021learning} provide effective ways to steer samples toward high-quality, semantically aligned outputs. Control is more subtle, however, since it can reflect different user intentions. In a standard text-to-image (T2I) generation task, a user may provide a single prompt and want a diverse set of plausible images; in this
  case, the desired control is over the diversity of the generated gallery. In other settings, such as e-commerce or product editing, the user may instead want to preserve a specific object
  while varying its surrounding context. Here, the desired control is to keep the product faithful to the input image
  while generating natural backgrounds that blend smoothly with the preserved object.

A natural way for controlling the output of a diffusion model is to control its initial noise. Recent studies suggest that initial noise affects generated images in structured and partly predictable ways. \citet{xu2025good} show that different initial noises can systematically affect output quality, layout, and object position. \citet{ban2025the} identify ``trigger patches'' in initial noise that help predict object locations in generated images. \citet{song2025ccs} report a highly linear relationship between perturbations to the initial noise and changes in the output. \citet{jia2026antithetic} show that opposite noise pairs yield significantly negatively correlated outputs. In parallel, a growing body of inference-time
methods optimizes or selects initial noises to improve alignment, aesthetics, diversity, or inverse
problem performance.

These works show that initial noise is not just a source of randomness; it is an effective
test-time control variable. However,  existing methods treat the noises in a generated batch as
separate objects. Each noise may be selected, ranked, or optimized independently. This leaves open a basic question: before any
per-sample optimization is performed, can we design the dependence structure among the initial
noises themselves? 

We answer this question through a coupling framework for initial-noise design. A coupling specifies
a joint distribution over a batch of random variables while preserving their marginals. In
the diffusion setting, this means constructing a batch of noises $(Z_1,\ldots,Z_K)$ such that each
$Z_i$ remains standard Gaussian, but the noises need not be independent. The pretrained diffusion
model therefore receives the same single-sample input distribution as usual, while the user can choose
how samples  are related.

This view turns coupling into a flexible control layer underneath existing initial-noise methods. It
naturally supports three modes of use. 
First, coupled noise can be used \textit{directly} as a  plug-in that replaces independent noise at the start of sampling, with no change to the pretrained model, prompt, sampler, guidance rule, or denoising trajectory. 
Second, coupled noise can serve as an \textit{initialization for direct noise optimization} since our framework allows the same optimizer to start from a structured coupling instead of from independent initialization. 
Third, one may perform \emph{amortized optimization of the coupling}, where the coupling is represented by a matrix in a Gaussian parameterization and optimized under marginal-preserving constraints. Once learned, such a coupling can be reused for new prompts by simply sampling from it, adding no extra denoising or test-time optimization cost.

Our framework naturally generalizes the antithetic sampling method of
\citep{jia2026antithetic}, which corresponds to our specific case for \(K=2\) and \((Z_1, Z_2) = (Z,-Z)\). This broader perspective allows us to scale to arbitrary $K$ and unlocks new usages, yielding (i) it extends to arbitrary $K$, introducing a $K$-sample repulsive coupling that is minimax optimal for worst-pair correlation; (ii) it enables new applications via subspace and learned couplings for structured generative tasks (e.g., fixed-object backgrounds); and (iii) it naturally fits into noise optimization pipelines as an initialization step, which is unaddressed by \cite{jia2026antithetic}.

Our contributions are as follows. First, we introduce coupled initial-noise design as a marginal-preserving
batch-level control layer for diffusion generation. Unlike standard seed selection or noise optimization,
which acts on realized noises, our framework designs the joint distribution from which the noises are
drawn. This makes coupling compatible with three common use cases: zero-cost sampling, structured
initialization for test-time noise optimization, and amortized learning of reusable coupling matrices. Second, we study Gaussian couplings in this framework. We identify useful special cases, including
identical, independent, antithetic, and repulsive couplings; give a minimax characterization of the
repulsive coupling; and develop a matrix parameterization for constructing and learning couplings
under marginal-preserving constraints. Third, we show that this design choice has practical value. On diverse text-to-image generation with
SD1.5, SDXL, and SD3, repulsive coupling matches or exceeds recent test-time optimization baselines
on several diversity metrics with no extra denoising or test-time optimization cost, while largely
preserving prompt alignment and image quality. When combined with contrastive noise optimization,
it often gives further gains. For fixed-object background generation, subspace coupling supports a
different form of structured control, preserving the object while varying the background.

\section{Related Work}
\label{sec:related work}
\textbf{Initial noise:} As the input to a diffusion model, initial noise has been studied both as a lens for understanding the input--output mechanism and as a lever for improving generation. Besides the mechanism studies mentioned in Section \ref{sec:intro}, optimizing initial noise has recently attracted interest for its potential to support inference-time scaling \citep{ma2025scaling}. These techniques have been shown to improve performance across many applications, including text-image alignment, aesthetics, diversity, and inverse problem solving; see \citep{guo2024initno, eyring2024reno, qi2024not, ben2024d, tang2025inferencetime, zhou2025golden, chihaoui2024blind, wang2024dmplug, wan2025fmplug, jia2026weak}. 

\textbf{Coupling} constructs random variables on the same probability space with prescribed marginal distributions, while allowing them to be dependent. It has played an important role in probability and statistics, especially in the study of Markov chains \citep{harris1955chains}. Couplings are among the most fundamental tools for analyzing convergence rates of Markov chains \citep{levin2017markov}. They are also used as algorithmic devices for designing sampling methods \citep{propp1996exact,neal2017circularly}, diagnosing MCMC convergence \citep{johnson1998coupling,biswas2019estimating}, and constructing unbiased estimators \citep{glynn2014exact,jacob2020unbiased,deligiannidis2024importance}. See \citet{lindvall2002lectures} for coupling as a proof technique and \citet{atchade2024unbiased} for coupling-based algorithms.

\section{Coupled Initial Noise}
\label{sec:coupling}

\subsection{Coupling of random variables}
Given probability distributions \(p_1,\dots,p_k\) on the same measurable space \((E,\mathcal E)\), a coupling is a probability distribution \(\gamma\) on  the product space \(E^k\) such that, if \((X_1,\dots,X_k)\sim \gamma\), then each marginal satisfies \(X_i \sim p_i\) for all \(i=1,\dots,k\). In other words, a coupling describes a joint distribution whose marginals are fixed to be $p_1, p_2, \ldots, p_k$. The set of all such couplings is denoted by \(\Gamma(p_1,\dots,p_k)\). As a simple example, one can couple two Bernoulli random variables \(\mathsf{Ber}(0.2)\) and \(\mathsf{Ber}(0.8)\) as follows. First sample \(U \sim \mathrm{Unif}[0,1]\), and then set $
X = \mathbf 1\{U \leq 0.2\},   Y = \mathbf 1\{U \leq 0.8\}.$

Given \(k\) standard Gaussian random variables in \(\mathbb R^d\), a coupling constructs a batch of $k$ random variables $(Z_1,\ldots,Z_k)$
such that each  $Z_i$ remains standard Gaussian,
while allowing dependence across different indices. Figure~\ref{fig:couplings} provides visualizations for the case \(k=2\) and \(d=1\).

  \begin{figure}[htbp!]
    \centering
    \includegraphics[width=\textwidth]{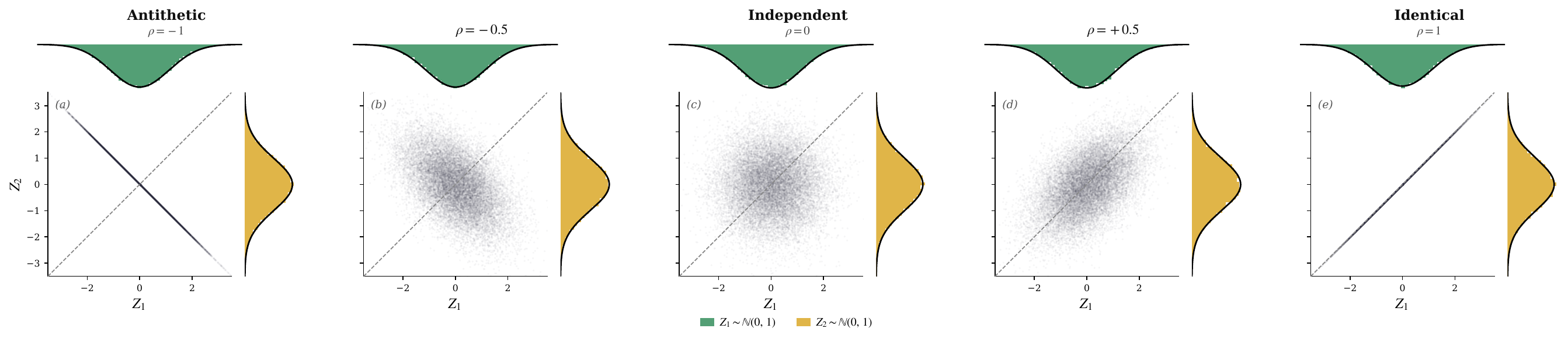}
    \caption{Five couplings of $Z_1, Z_2 \sim \mathbb{N}(0, 1)$, illustrated via
  $20{,}000$
      joint samples. Each panel shows the joint scatter of $(Z_1, Z_2)$ with the
  marginal
      density of $Z_1$ (green, top) and $Z_2$ (amber, right); the dashed line marks
  $Y = X$.
      \textbf{(a) Antithetic} ($\rho = -1$): $Z_2 = -Z_1$, achieving maximal
  negative
      correlation.
      \textbf{(b)} $\rho = -0.5$: an intermediate negatively correlated Gaussian
   coupling.
      \textbf{(c) Independent} ($\rho = 0$): $Z_1$ and $Z_2$ are drawn
  independently.
      \textbf{(d)} $\rho = +0.5$: an intermediate positively correlated
  coupling.
      \textbf{(e) Identical ($\rho = 1$)}: $Z_1 = Z_2$, all mass
  collapses
      onto the diagonal $Z_1 = Z_2$.}
    \label{fig:couplings}
  \end{figure}

\subsection{Coupled Gaussian Noise}

Rather than sampling the initial noises independently, we view the whole batch $(Z_1,\ldots,Z_k)$ as a jointly sampled random vector, while requiring each marginal to remain standard Gaussian. Thus, each sample is generated from the usual initial-noise distribution, but the dependence structure across the batch can be chosen by design. This perspective is especially natural in practical image-generation systems, where a single prompt often produces a small set of candidate images rather than one sample. Since sampling is stochastic, different initial noise vectors can lead to different outputs for the same prompt, allowing users or automatic ranking modules to select a preferred candidate.

We study different couplings of the Gaussian initial noises. A few couplings useful for our paper are:

\begin{itemize}[leftmargin = *]
        \item \textit{Identical coupling}: $Z\sim \bN(0, I_d)$ and $Z_1 = Z_2 = \ldots = Z_k = Z$.

    \item \textit{Independent coupling}: 
    $(Z_1, \ldots, Z_k)$ are independent draws with $Z_i \sim \bN(0, I_d)$.
    This serves as the standard baseline with no dependence across samples.

    \item \textit{Antithetic coupling}:
    For $k = 2$, set $(Z_1, Z_2) = (Z, -Z)$ where $Z \sim \bN(0, I_d)$.
    More generally, one can construct collections with strong negative correlations.
This coupling induces strong negative dependence and is classically used for variance reduction; in diffusion generation, it can also encourage more diverse paired outputs.

    \item \textit{Repulsive coupling}: Algorithm \ref{alg:repulsive-coupling} generates negatively correlated Gaussian samples satisfying $\operatorname{Cov}(z_i, z_j) = -\frac{1}{k-1}\, I_d,$ see Lemma \ref{lem:repulsive-coupling} for a formal statement and proof. It induces the same  correlation between every pair of samples, and reduces to the antithetic coupling when $k = 2$. 
\begin{algorithm}[htbp!]
\caption{Repulsive Coupled Initial Noise Sampling}
\label{alg:repulsive-coupling}
\begin{algorithmic}[1]
\Require Gallery size $k \ge 2$, noise dimension $d$
\Ensure $K$ initial noises that are all marginally standard Gaussian 

\State Sample independent base noises $
u_1,\ldots,u_K \overset{\mathrm{i.i.d.}}{\sim} \bN(0,I_d).$

\State Compute the batch mean $\bar u \gets k^{-1}\sum_{j=1}^k u_j.$

\State Construct repulsively coupled initial noises
\[
z_i \gets \sqrt{\frac{k}{k-1}}\,(u_i-\bar u),
\qquad i=1,\ldots,k .
\]
\State \Return $(z_1,\ldots,z_k)$
\end{algorithmic}
\end{algorithm}

\end{itemize}

The identical, independent, and repulsive couplings all belong to the equicorrelated Gaussian family
\[
    \operatorname{Cov}(z_i,z_j)
    =
    \begin{cases}
    I_d, & i=j,\\
    c I_d, & i\ne j,
    \end{cases}
\]
with $c = 1, 0, -1/(k-1)$ respectively. The following proposition shows equicorrelated coupling exists whenever $c \in [-1/(k-1), 1]$.
\begin{proposition}[Equicorrelated Gaussian coupling and its minimax property]
\label{prop:equicorrelated-coupling}
Let \(k\ge 2\) and \(d\ge 1\). For a scalar \(c\in\mathbb R\), there exists a Gaussian coupling of $K$ standard Gaussian vectors in $\mathbb R^d$ satisfying 
\[
    \operatorname{Cov}(z_i,z_j)
    =
    \begin{cases}
    I_d, & i=j,\\
    c I_d, & i\ne j,
    \end{cases}
\]
if and only if $-1/(k-1)\le c\le 1.$ Moreover, the endpoint \(c=-1/(k-1)\) is the most repulsive possible coupling in
the following minimax sense. Let \(\Gamma_k\) denote the set of all couplings of
\(k\) standard Gaussian vectors in \(\mathbb R^d\). Then
\[
    \inf_{\gamma\in\Gamma_k}
    \max_{1\le i<j\le k}
    \frac{1}{d}\mathbb E_{\gamma}\langle z_i,z_j\rangle
    =
    -\frac{1}{k-1}.
\]
The equicorrelated Gaussian coupling with \(c=-1/(k-1)\) attains this value. 
\end{proposition}

Proposition~\ref{prop:equicorrelated-coupling} shows that the repulsive
coupling is the most spread-out joint design for \(k\) initial noises in the
minimax pairwise-correlation sense. This property makes it especially useful
for the diverse generation experiments studied later. Indeed, the repulsive coupling minimizes a broad class of feature-similarity objectives called RBF similarity; see Appendix~\ref{subsec:optimality} for the statement and proof.

We also study what coupling can and cannot change, and compare its effects with i.i.d. noise.

\begin{proposition}[What coupling can and cannot change]
\label{prop:marginal-invariance}
Fix a prompt \(p\) and a deterministic generator \(G\). Let
\((Z_1,\ldots,Z_K)\) be any coupling such that \(Z_i\sim \bN(0,I_d)\) for
each \(i\), and define \(X_i=G(p,Z_i)\). Then each \(X_i\) has the same
marginal distribution as the output generated from an independent standard
Gaussian noise \(Z\sim \bN(0,I_d)\). Consequently, for any integrable
single-image score \(q\),
$    \bE\left[K^{-1}\sum_{i=1}^K q(p,X_i)\right]
    =
    \bE_{Z\sim \bN(0,I_d)}\left[q(p,G(p,Z))\right].$
\end{proposition}
Thus, changing the coupling cannot change the expected value of an averaged
single-image metric, such as average image quality or average prompt alignment,
as long as the generator and prompt are fixed. It can, however, change any
metric that depends on two or more samples jointly, such as pairwise distance,
similarity, diversity, or other gallery-level objectives.
To understand the direction of this gallery-level effect, we give a quantitative result in Theorem \ref{thm:coupling-taylor}. Informally,
\[
    \text{effect of coupling}
    \approx
    \sum_{i<j}
    R_{ij}\,
    \mathbb E_{\mathrm{iid}}
    \left[
        \operatorname{tr}\!\left(\nabla^2_{ij} H_p\right)
    \right],
\]
where \(H_p\) is the metric of interest. The formal statement and proof is deferred to the appendix. The proof interpolates between independent and coupled Gaussians along $R_t = (1-t)I_K + tR$ and applies Gaussian integration by parts.

\textbf{Local linear interpretation.}
The effect of repulsive coupling can also be seen from a first-order approximation of the generator. Fix a prompt \(p\), and let \(G_p\) denote the mapping from the initial noise to the generated image. Around a typical noise value, write $G_p(z) \approx a_p + J_p z,$
where \(J_p\) is the local Jacobian. Suppose \(Z_i\) and \(Z_j\) are marginally standard Gaussian and satisfy $
\operatorname{Cov}(Z_i,Z_j)=cI .$
Then
\[
\mathbb{E}\|G_p(Z_i)-G_p(Z_j)\|_2^2
\approx
\mathbb{E}\|J_p(Z_i-Z_j)\|_2^2
=
2(1-c)\|J_p\|_2^2 .
\]
Thus, among couplings with the same Gaussian marginals, making \(c\) more negative increases the expected pairwise separation. Independent sampling corresponds to \(c=0\), while the repulsive equicorrelated coupling uses \(c=-1/(K-1)\), giving $
\mathbb{E}\|G_p(Z_i)-G_p(Z_j)\|_2^2
\approx
2\frac{K}{K-1}\|J_p\|_2^2 .$
This is a factor of \(K/(K-1)\) larger than the corresponding local prediction under independent sampling.

\subsection{Parametrizing Gaussian Couplings}\label{subsec:optimize}
The constructions above are particular choices within a broader design space
of couplings. Rather than specifying the dependence structure by hand, we can parametrize a
large class of valid noise couplings and then optimize over the resulting
parameters. The following theorem gives a matrix parametrization for
Gaussian couplings.

\begin{theorem}[Matrix parametrization of Gaussian noise couplings]
\label{thm:coupling parametrization}
Let \(u_1,\ldots,u_r \overset{\mathrm{i.i.d.}}{\sim} \bN(0,I_d)\), and let
\(A\in \mathbb{R}^{M\times r}\) be a matrix with unit Euclidean norm on each
row, i.e.,   $ \sum_{\ell=1}^r A_{m\ell}^2 = 1$ for each $m$.
Writing \(Z=[z_1,\ldots,z_M]^\top\)
\(U=[u_1,\ldots,u_r]^\top\), and set  \(Z=AU\). Then each
\(z_m\sim \bN(0,I_d)\). Hence \((z_1,\ldots,z_M)\) defines a valid coupling
of \(M\) standard Gaussian noise vectors in \(\mathbb{R}^d\). Moreover, for
every pair \(i,j\),
\[
    \operatorname{Cov}(z_i,z_j)
    =
    \left(\sum_{\ell=1}^r A_{i\ell}A_{j\ell}\right) I_d .
\]

Conversely, suppose \((z_1,\ldots,z_M)\) is a centered jointly Gaussian
coupling such that each \(z_m\sim \bN(0,I_d)\) and $\operatorname{Cov}(z_i,z_j)=\rho_{ij}I_d.$
Let \(R=(\rho_{ij})_{i,j=1}^M\). If \(\operatorname{rank}(R)\le r\), then
there exists a matrix \(A\in\mathbb{R}^{M\times r}\) with unit-norm rows such
that \(R=AA^\top\). With this choice of \(A\), the construction above has the
same joint distribution as \((z_1,\ldots,z_M)\).
\end{theorem}

When \(d=1\) and \(r=M\), Theorem~\ref{thm:coupling parametrization} has a
direct geometric interpretation. Every Gaussian coupling of \(M\)
standard normal random variables can be obtained by drawing
\(u\sim \bN(0,I_M)\) and setting \(z=Au\), where \(A\in\mathbb R^{M\times M}\)
has unit Euclidean norm on each row. Thus, choosing such a Gaussian coupling is
equivalent to choosing \(M\) unit vectors. The independent coupling corresponds
to \(A=I_M\), while the repulsive coupling corresponds to $
    A=\sqrt{\frac{M}{M-1}}
    \left(I_M-\frac{1}{M}\mathbf 1\mathbf 1^\top\right).$

An image is often very high-dimensional. For example, a
\(256\times256\times 3\) image has roughly \(200{,}000\) coordinates. In principle, designing a coupling of \(M\)
initial noises in dimension \(d\) amounts to specifying a joint distribution on
\(\mathbb R^{Md}\) whose \(M\) marginals are all \(\bN(0,I_d)\). A full Gaussian
parametrization would therefore require an \(Md\times Md\) matrix,
which is computationally infeasible at typical image dimensions. In practice, we focus on structured families that are much lower-dimensional
than a full \(Md\times Md\)  matrix. We describe two useful
parametrizations below.

\textbf{Same sample-level dependence:} First, one can use a coupling that shares the same sample-level dependence
across all coordinates, i.e.,  $\operatorname{Cov}(z_i,z_j)= c_{ij} I_d.$ This is the class covered by Theorem~\ref{thm:coupling parametrization} with
\(r=M=K\) and \(N=d\). In this case, the coupling is specified by a much smaller
\(K\times K\) sample-level matrix, rather than a full \((Kd)\times(Kd)\) pixel-level matrix.

\textbf{Coupling on a subspace:} We can also couple the noises selectively on a subspace of interest, while using
a simple default coupling on its complement. Let \(V\subset \mathbb R^d\) be a
subspace with orthogonal projection \(P_V\), and let \(P_{V^\perp}=I-P_V\). The
subspace \(V\) may represent a color channel, a low-dimensional image statistic,
a semantic direction, or a region associated with an object. We write
\[
    z_i = P_V z_i + P_{V^\perp} z_i .
\]
On \(V\), the coupling can be chosen to reflect the intended control: identical
coupling for preservation, repulsive coupling for diversity, or a learnable
coupling matrix \(A\) for optimization. On \(V^\perp\), we keep a simple fixed
coupling, for example an equicorrelated coupling or, in particular, the
independent coupling. Since the two components lie in orthogonal subspaces and both have standard
Gaussian marginals on their respective subspaces, each full noise vector still
satisfies $z_i \sim \mathbb N(0, I_d)$. This parametrization lets us optimize the dependence structure only in the
coordinates that are most relevant to the target behavior, while keeping a
simple fixed coupling on the remaining directions.

\subsection{Coupling-Based Noise Design and Optimization}
\label{subsec:optimization}
We next describe three ways to use these initial-noise couplings.
Concrete examples and experiments are provided in Sections~\ref{sec:applications} and~\ref{sec:experiments}.
\begin{itemize}[leftmargin=*]
    \item \textbf{Using coupled noise directly:} A carefully designed coupling can replace independent noises and improve several applications at no additional cost. For example, \cite{jia2026antithetic} shows that antithetic coupling can yield \(10\)--\(100\times\) efficiency gains when estimating downstream statistics. Our experiments also suggest that repulsive coupling alone can roughly match the diversity achieved by optimization-based methods such as \cite{kim2026diverse}, while being much faster.
    \item \textbf{Coupled initialization for direct noise optimization:} A direct way to optimize a gallery-level reward is to optimize the realized
initial noises themselves. Let \(\Gamma\) denote the set of all couplings of
\(K\) standard Gaussian random variables. Given a prompt \(p\), a reward
function \(\mathcal R\), and an initial coupling \(\gamma\in\Gamma\), we first
draw $z^{(0)}_{1:K} \sim \gamma$. One can then solve
\begin{align}
    z_{1:K}^{\star}
    \in
    \arg\max_{z_1,\ldots,z_K}
    \mathcal R_p\big(
        G(p,z_1),\ldots,G(p,z_K)
    \big),
    \label{eq:direct-noise-opt}
\end{align}
possibly with additional regularization on the noises. In this use case, the coupling only specifies the initialization of the
optimization procedure. Existing inference-time noise optimization methods
\citep{guo2024initno,eyring2024reno,qi2024not,kim2026diverse} naturally fit
within this framework by using the independent coupling. The coupling viewpoint then adds an additional design
choice: the same test-time optimizer can instead be initialized from a structured
coupling without changing the rest
of the algorithm.
\item \textbf{Amortized optimization of the coupling matrix:} Instead of directly optimizing the noises, we can seek a coupling
\(\gamma^\star\in\Gamma\) that maximizes the expected distribution-level reward:
\begin{align}
    \gamma^\star
    \in
    \arg\max_{\gamma\in\Gamma}
    \;
    \mathbb E_{p\sim \mathcal P}
    \mathbb E_{z_{1:K}\sim \gamma}
    \left[
        \mathcal R_p\big(
            G(p,z_1),\ldots,G(p,z_K)
        \big)
    \right].
    \label{eq:coupling-level-opt}
\end{align}
Unlike direct noise optimization that operates on \textit{sample level}, this objective does not optimize a particular
realization of the initial noises. Instead, it optimizes the \textit{joint distribution}
from which the initial noises are sampled. Therefore, once \(\gamma^\star\) is
learned, it can be reused for new prompts by simply sampling
\(z_{1:K}\sim \gamma^\star\).
The matrix parametrization in Theorem~\ref{thm:coupling parametrization}
provides a tractable way to implement this idea.  Appendix~\ref{app:subsec:diverse} shows that amortized optimization recovers
the repulsive coupling for gallery diversity. Appendix~\ref{subsec:brightness}
gives a second example, where the coupling matrix is learned for a
brightness-clustered gallery objective. This objective asks for a structured
split into two darker and two brighter samples, rather than generic diversity.
The learned coupling improves the brightness split over both independent and
repulsive couplings, suggesting that amortized coupling learning can capture
simple gallery-level structure beyond diversity.
\end{itemize}

\section{Applications}\label{sec:applications}

For all applications in this section, we treat the generator as a deterministic map. Conditional on the initial noise \(z\) and possibly a prompt $p$, the output is fixed; hence all randomness comes from the initialization. This setting covers deterministic diffusion samplers such as DDIM, probability-flow ODE samplers, and flow-matching models \citep{song2021denoising,lipman2023flow}.

\subsection{Diverse Generation}
\label{subsec:diverse}

An important use case of coupled initial noise is diverse generation. 
Given a prompt \(p\), a generator \(G\), and a desired gallery size \(K\), the goal is to generate $K$ images that are diverse within the batch, while each image still maintains fidelity, visual quality, and prompt alignment. As discussed in Section~\ref{sec:intro}, independently sampled noises can produce redundant outputs \citep{ban2025the}.

Our idea is to use the most spread-out initial-noise batch to encourage diverse outputs. Therefore, we sample $(z_1,\ldots,z_K)$
from the repulsive Gaussian coupling in Algorithm \ref{alg:repulsive-coupling}, so that each \(z_i\sim \bN(0,I_d)\), while for \(i\neq j\), $\operatorname{Cov}(z_i,z_j)
    =
    -\frac{1}{K-1}I_d.$
We then generate the gallery by $x_i = G(p,z_i)$ for each $i$. Since each marginal noise remains standard Gaussian, each image is generated from the usual input distribution of the pretrained diffusion model. The only change is the dependence structure across the batch. Thus, repulsive coupling can increase gallery-level variation without modifying the prompt, model, sampler, guidance rule, or denoising procedure.

This gives a training-free and optimization-free method for diverse generation. It can also be combined with test-time diversity optimization methods by using the repulsive coupling as the initialization instead of independent Gaussian noise. We evaluate both uses in Section~\ref{subsec:diverse T2I}.

\subsection{Background Generation with a Fixed Object}
  Another useful application is controlled background generation while preserving a specified foreground object. This task also illustrates a second use of coupled initial noise: the coupling supplies a structured initialization that can be directly combined with noise optimization, rather than replacing optimization altogether. Given an input image $x^{\star}$, a foreground mask $m\in\{0,1\}^{H\times W}$, and optionally a prompt $p$, the goal
  is to produce a gallery with diverse backgrounds while keeping the masked foreground faithful to $x^{\star}$ (see Figure~\ref{fig:bedroom-object} for illustration). This
  setting involves three desiderata: 1. the foreground object should be preserved, 2. the background should vary across the
  gallery, and 3. the resulting images should remain natural, without visible artifacts at the object-background
  boundary.

Coupling the foreground noise identically does not by
  itself guarantee exact foreground preservation after denoising. We therefore view the task as an inpainting
  problem: the foreground is the known region to be preserved, while the background is the region to be resampled.
  Starting from the coupled initialization $z_i^{(0)}$, we refine only the foreground component by optimizing the initial noise, following the diffusion inverse-problem
  strategy of \citet{wang2024dmplug,jia2026weak}.
  \begin{align}
  \min_{z_1, \ldots, z_k}
  \sum_{i=1}^K
  \left||
  m\odot \big(G(p,z_i)-x^{\star}\big)
  \right||_2^2,
  \qquad
  \text{such that }\quad
  (1-M)\odot z_i=(1-M)\odot z_i^{(0)},
  \label{eqn:fixed-object-opt}
  \end{align}
where \(\odot\) denotes elementwise multiplication.
  The constraint keeps the repulsively coupled background noise fixed, so the optimization only improves foreground
  fidelity.

  This procedure is training-free and can be applied to both unconditional diffusion models and text-conditioned
  models. For latent diffusion models, the image-space mask is downsampled to the latent
  resolution when constructing \(M\), while the fidelity loss is evaluated in image space. 
\section{Experiments}\label{sec:experiments}

\subsection{Diverse Text-to-Image Generation}\label{subsec:diverse T2I}
Our setup follows \cite{kim2025diverse}. We evaluate 2,000 prompts sampled from COCO captions
  using Stable Diffusion v1.5 (SD1.5), Stable Diffusion XL (SDXL), and Stable Diffusion 3 (SD3). For each model and
  prompt, we generate three images under four sampling settings: independent noise, independent noise with
  Contrastive Noise Optimization (CNO) \citep{kim2025diverse}, coupled noise, and coupled noise with CNO. We measure
  diversity using average pairwise L2 distance in pixel space, Mean Pairwise Similarity (MSS), and Vendi score \cite{friedman2023the}, and evaluate prompt
  alignment and perceptual quality using CLIP score \citep{radford2021learning} and PickScore \citep{kirstain2023pick}. 
  Results are reported in
  Table~\ref{tab:diversity_quality_metrics}, and qualitative comparisons are shown in Figure~\ref{fig:diverse_img},\ref{fig:diverse_img_sdxl}. To establish statistical significance, we also include standard deviations and p-values for all diversity metrics.
  Appendix~\ref{app:subsec:diverse} provides details and additional experiments, including ablation studies on both the gallery size and the correlation strength, an amortized optimization experiment, and a compute-cost comparison.
Repulsive coupling has essentially the same wall-clock generation cost as i.i.d. sampling for a $K=3$ gallery, while CNO adds a modest optimization overhead  in our setting (see Table \ref{tab:diverse_compute_cost}) with only three optimization steps.
We summarize the main findings below.

  First, coupled noise improves diversity compared with independent noise. Table~\ref{tab:diversity_quality_metrics} shows that coupled sampling
  improves the diversity metrics across all models. For example, on SD3, coupled noise improves L2 from
  $0.1417$ to $0.1667$, reduces MSS from $0.3007$ to $0.2827$, and improves Vendi from $2.7256$ to $2.7532$. Importantly, these diversity gains do not come at the cost of prompt alignment or perceptual quality. Differences across the L2, MSS, and Vendi metrics are highly statistically significant ($p$-value $< 10^{-7}$ for all metric $\times$ model combinations). 
CLIP score
  and PickScore remain nearly unchanged across independent and coupled sampling. This supports the view that noise coupling mainly changes the joint distribution of a multi-sample gallery, increasing variation across samples while largely preserving the marginal quality and text alignment of each sample.

  Second, coupled noise provides a cost-free alternative to optimization-based methods. Unlike CNO,
  which requires test-time optimization, coupled noise preserves the same sampling cost as independent generation. Despite requiring no extra cost, coupled noise often matches or outperforms i.i.d. + CNO on diversity metrics. For example, on SD1.5,
  coupled noise achieves better L2, MSS, and Vendi than i.i.d. + CNO. This suggests that much of the diversity gain can be obtained by modifying the joint sampling distribution, without adding an expensive optimization loop.

Third, coupled noise is complementary to CNO. Combining coupled noise with CNO often yields the strongest 
  performance. On SDXL, coupled + CNO achieves the best L2, MSS, and Vendi among all methods, with a highly significant L2
  improvement over i.i.d. + CNO ($p< 10^{-290}$); MSS and Vendi also improve, although the improvements are more modest. On SD3, coupled + CNO
  improves L2 from $0.1612$ to $0.1844$, reduces MSS from $0.2804$ to $0.2675$, and improves Vendi from $2.7554$ to
  $2.7749$, with all three changes statistically significant ($p <10^{-16}$).
  As a result, coupled noise can
  provide additional gains on top of CNO with minimal changes to prompt alignment and quality.
  \begin{figure}[htbp!]
    \centering
    \includegraphics[width=\textwidth]{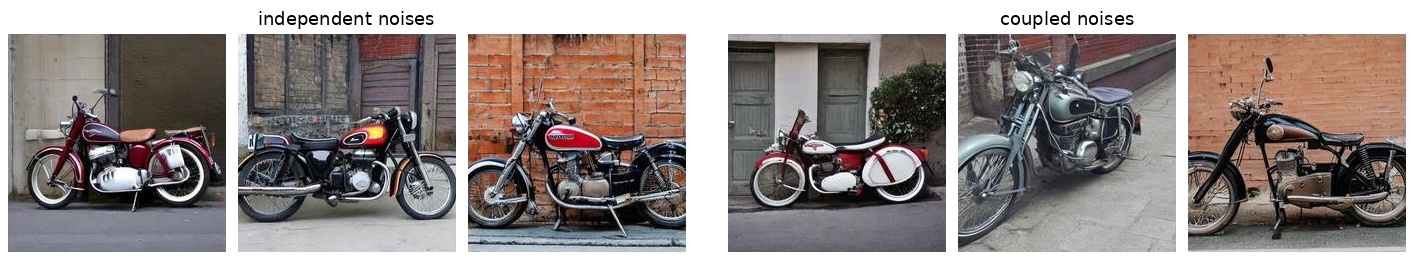}
    \caption{Independent versus coupled noises with the same prompt ``A classic motorcycle in a parking space along the street.'' The coupled-noise  samples  exhibit greater visual
  diversity across pose, viewpoint, object scale, and scene layout while remaining consistent with the prompt.}
    \label{fig:diverse_img}
  \end{figure}

Finally, we examine whether amortized coupling optimization can either improve upon the analytic repulsive coupling or recover it from data. We parameterize a Gaussian coupling using the row-normalized matrix form in Theorem~\ref{thm:coupling parametrization}, and optimize this matrix under several diversity-based objectives. Across random initializations, the learned sample-level correlation matrix consistently approaches the repulsive coupling. This suggests that, for diversity-oriented gallery generation, the minimax repulsive coupling is not only a convenient analytic construction, but also the solution favored by amortized optimization. Experiment details are provided in Appendix \ref{app:subsec:diverse}.

\begin{table}[htbp!]
  \centering
  \caption{Comparison of diversity and quality metrics across different diffusion models. We report standard
  deviations in parentheses for diversity metrics.}
  \label{tab:diversity_quality_metrics}
      \resizebox{0.95\linewidth}{!}{
  \begin{tabular}{llcccccc}
  \toprule
  Model & Method & Opt.
  & \multicolumn{3}{c}{Diversity metrics}
  & \multicolumn{2}{c}{Quality metrics} \\
  \cmidrule(lr){4-6} \cmidrule(lr){7-8}
  & & & L2 $\uparrow$ & MSS $\downarrow$ & Vendi $\uparrow$ & CLIP $\uparrow$ & PickScore $\uparrow$ \\
  \midrule

  \rowcolor{blue!8}
  SD1.5 & i.i.d.         & \xmark & 0.126 (0.04) & 0.150 (0.07) & 2.913 (0.07) & 0.662 & 21.407 \\
  \rowcolor{blue!8}
        & i.i.d.+CNO   & \cmark & 0.136 (0.04) & 0.148 (0.07) & 2.916 (0.07) & 0.661 & 21.402 \\
  \rowcolor{blue!8}
        & coupled        & \xmark & \underline{0.156} (0.04) & \textbf{0.141} (0.07) & \textbf{2.921} (0.07) & 0.662
  & 21.391 \\
  \rowcolor{blue!8}
        & coupled+CNO  & \cmark & \textbf{0.163} (0.05) & \underline{0.144} (0.07) & \underline{2.920} (0.07) &
  0.660 & 21.399 \\
  \midrule

  \rowcolor{green!8}
  SDXL  & i.i.d.         & \xmark & 0.087 (0.02) & 0.208 (0.08) & 2.853 (0.10) & 0.670 & 22.474 \\
  \rowcolor{green!8}
        & i.i.d.+CNO   & \cmark & \underline{0.106} (0.03) & \underline{0.187} (0.07) & \underline{2.879} (0.09) &
  0.668 & 22.419 \\
  \rowcolor{green!8}
        & coupled        & \xmark & 0.105 (0.03) & 0.198 (0.08) & 2.866 (0.09) & 0.669 & 22.459 \\
  \rowcolor{green!8}
        & coupled+CNO  & \cmark & \textbf{0.122} (0.03) & \textbf{0.184} (0.08) & \textbf{2.881} (0.09) & 0.668 &
  22.412 \\
  \midrule

  \rowcolor{orange!10}
  SD3   & i.i.d.         & \xmark & 0.142 (0.05) & 0.301 (0.10) & 2.726 (0.15) & 0.664 & 22.524 \\
  \rowcolor{orange!10}
        & i.i.d.+CNO   & \cmark & 0.161 (0.05) & \underline{0.280} (0.09) & \underline{2.755} (0.15) & 0.664 &
  22.505 \\
  \rowcolor{orange!10}
        & coupled        & \xmark & \underline{0.167} (0.06) & 0.283 (0.09) & 2.753 (0.14) & 0.664 & 22.529 \\
  \rowcolor{orange!10}
        & coupled+CNO  & \cmark & \textbf{0.184} (0.06) & \textbf{0.267} (0.09) & \textbf{2.775} (0.14) & 0.664 &
  22.515 \\
  \bottomrule
  \end{tabular}}
  \end{table}
 
  \subsection{Fixed-Object Background Generation}\label{subsec:object-exp}
  Unlike the diverse-generation experiments, this experiment uses coupling as a structured initialization for a constrained noise-optimization procedure: the background noise remains repulsively coupled to promote variation, while the foreground component is identically coupled and then optimized.
  
We use 100 bedroom images from the LSUN bedroom dataset~\citep{yu2015lsun}, with bed masks produced by YOLOv8
~\citep{Jocher_Ultralytics_YOLO_2023} as foreground objects (see the first row of Figure~\ref{fig:bedroom-object}). For each image, we generate a
gallery of $K$ images and evaluate four aspects: background diversity, measured by pairwise LPIPS
\citep{zhang2018unreasonable} and L2 on the background region; foreground preservation, measured by RMSE on the
masked region; boundary smoothness, measured by Boundary-TV in a narrow band around the mask; and global naturalness,
measured by BRISQUE~\citep{mittal2012no}.

We evaluate fixed-object background generation with Stable Diffusion 1.5, using $K=8$ samples and $10$ optimization
steps for our method. We compare against two latent-space baselines: DDIM inversion of the full image (DDIM inv.)
and DDIM-inverted noise inside the foreground mask with random background noise (Object-inv.). We also compare against
SDEdit~\citep{meng2022sdedit} and Stable Diffusion inpainting~\citep{rombach2022high}. Results are summarized in Table \ref{tab:fix-obj-metric}.

DDIM inv. reconstructs the
input from its inverted noise and therefore has no  background diversity. Object-inv. gives high background
diversity and low BRISQUE, but its foreground error and boundary roughness are much worse, showing that fixing only
the object latent is insufficient. SDEdit preserves the foreground by construction (RMSE $0.01$), but
yields limited background diversity (L2-B $0.02$) and worse BRISQUE ($35.99$). Stable Diffusion inpainting is a strong
task-specific baseline. Compared with it, our method uses no task-specific inpainting model and no additional training,
while reducing Boundary-TV from $0.022$ to $0.012$ and lowering BRISQUE from $34.07$ to $17.52$, with better
background diversity. The cost is a moderate increase in foreground RMSE ($0.09$ versus $0.05$). Appendix~\ref{app:subsec:fixed-object} provides experiment details, additional results on the effect of the number of optimization steps, and results on unconditional diffusion models; in particular, increasing the number of optimization steps further reduces foreground RMSE while preserving the naturalness advantage.
\begin{table}[t]
      \centering
      \caption{Fixed-object background generation with SD 1.5. We report mean (standard deviation).
LPIPS-B and L2-B measure background diversity; RMSE-F measures foreground preservation; Boundary-TV measures
local roughness near the boundary; BRISQUE measures global naturalness.}
\label{tab:fix-obj-metric}
      \resizebox{0.95\linewidth}{!}{
      \begin{tabular}{lccccc}
      \toprule
      Method & LPIPS-B $\uparrow$ & L2-B $\uparrow$ & RMSE-F $\downarrow$ & Boundary-TV $\downarrow$ & BRISQUE $
  \downarrow$ \\
      \midrule
      DDIM inv. & 0 (0) & 0 (0) & 0.09 (0.03) & 0.015 (0.006) & 45.40 (10.95) \\
      Object-inv. + random BG & \textbf{0.47} (0.16) & \textbf{0.05} (0.01) & 0.14 (0.03) & 0.023 (0.011) &
  \underline{18.16} (8.80) \\
      SDEdit & 0.31 (0.12) & 0.02 (0.01) & \textbf{0.01} (0.00) & 0.025 (0.007) & 35.99 (9.67) \\
      SD inpaint & 0.42 (0.14) & \textbf{0.05} (0.01) & \underline{0.05} (0.03) & \underline{0.022} (0.007) & 34.07
  (7.27) \\
      Ours & \underline{0.45} (0.16) & \textbf{0.05} (0.01) & 0.09 (0.03) & \textbf{0.012} (0.005) & \textbf{17.52}
  (9.57) \\
      \bottomrule
      \end{tabular}}
      \vspace{-1em}
  \end{table}
\section{Discussion}\label{sec:discussion}

Our results show that the joint distribution of initial noises is a useful control variable for diffusion generation. Coupled sampling can increase diversity, improve user-specified gallery objectives, and support fixed-object background generation. The method also supports flexible modes of use: it can be applied directly as a sampling plug-in, used to initialize noise optimization, or optimized amortized at the coupling level.
Several limitations and open directions remain. The coupling approach is mainly a batch-level control method, rather than to improve the quality of a single image.  Our learned couplings are currently limited to low-dimensional matrices and relatively simple gallery objectives, so richer semantic control may require more expressive parameterizations. Finally, our experiments focus on image diffusion models; whether coupled initial-noise design is useful for other modalities, such as video, audio, 3D generation, or scientific generative models, remains an open question.

\section*{Acknowledgement}
 Jing Jia and Guanyang Wang acknowledge support from the National Science Foundation through grants
DMS–2210849 and FET-2403007.
Liyue Shen acknowledges funding support by National Science Foundation (NSF) via grant IIS-2435746, Defense Advanced Research Projects Agency (DARPA) under contract No. HR00112520042, as well as the University of Michigan MIDAS PODS Grant Award.

\bibliography{couple_noise}
\bibliographystyle{chicago}
\appendix

\section{Mathematical Details and Proofs}
\label{app:math}

Throughout this appendix, \(K\) denotes the gallery size and \(d\) denotes the
dimension of each initial-noise vector. We write \(I_d\) for the \(d\times d\)
identity matrix and \(\mathbf 1_K\) for the all-ones vector in \(\mathbb R^K\).
All Gaussian random vectors are allowed to be degenerate when the covariance
matrix is positive semidefinite but not positive definite.

\subsection{Repulsive Gaussian coupling}

We first verify the Gaussian coupling used in Algorithm \ref{alg:repulsive-coupling}.

\begin{lemma}[Repulsive Gaussian coupling]
\label{lem:repulsive-coupling}
Let \(U_1,\ldots,U_K \overset{\mathrm{i.i.d.}}{\sim} \bN(0,I_d)\), define
\[
    \bar U := \frac{1}{K}\sum_{j=1}^K U_j,
    \qquad
    Z_i := \sqrt{\frac{K}{K-1}}\,(U_i-\bar U),
    \qquad i=1,\ldots,K .
\]
Then each \(Z_i\sim \bN(0,I_d)\). Moreover, for \(i\ne j\),
\[
    \operatorname{Cov}(Z_i,Z_j)
    =
    -\frac{1}{K-1} I_d .
\]
In addition,
\[
    \sum_{i=1}^K Z_i = 0
    \qquad \text{almost surely}.
\]
\end{lemma}

\begin{proof}
Since \(Z_i\) is a linear transformation of jointly Gaussian random vectors,
\((Z_1,\ldots,Z_K)\) is jointly Gaussian. Also,
\[
    \mathbb E[Z_i]=0 .
\]
For the marginal covariance, note that
\[
    \operatorname{Cov}(U_i,\bar U)=\frac{1}{K}I_d,
    \qquad
    \operatorname{Cov}(\bar U,\bar U)=\frac{1}{K}I_d .
\]
Therefore,
\begin{align*}
    \operatorname{Cov}(U_i-\bar U)
    &=
    \operatorname{Cov}(U_i)
    -2\operatorname{Cov}(U_i,\bar U)
    +\operatorname{Cov}(\bar U)
    \\
    &=
    I_d - \frac{2}{K}I_d + \frac{1}{K}I_d
    =
    \frac{K-1}{K}I_d .
\end{align*}
Multiplying by the factor \(K/(K-1)\) gives
\[
    \operatorname{Cov}(Z_i)=I_d .
\]
Thus \(Z_i\sim \bN(0,I_d)\).

For \(i\ne j\),
\[
    \operatorname{Cov}(U_i,U_j)=0,
    \qquad
    \operatorname{Cov}(U_i,\bar U)=\operatorname{Cov}(U_j,\bar U)
    =
    \frac{1}{K}I_d .
\]
Hence
\begin{align*}
    \operatorname{Cov}(U_i-\bar U, U_j-\bar U)
    &=
    \operatorname{Cov}(U_i,U_j)
    -\operatorname{Cov}(U_i,\bar U)
    -\operatorname{Cov}(\bar U,U_j)
    +\operatorname{Cov}(\bar U,\bar U)
    \\
    &=
    0-\frac{1}{K}I_d-\frac{1}{K}I_d+\frac{1}{K}I_d
    =
    -\frac{1}{K}I_d .
\end{align*}
After multiplying by \(K/(K-1)\), we obtain
\[
    \operatorname{Cov}(Z_i,Z_j)
    =
    -\frac{1}{K-1}I_d .
\]
Finally,
\[
    \sum_{i=1}^K Z_i
    =
    \sqrt{\frac{K}{K-1}}
    \sum_{i=1}^K (U_i-\bar U)
    =
    \sqrt{\frac{K}{K-1}}
    \left(
    \sum_{i=1}^K U_i
    -
    K\bar U
    \right)
    =
    0 .
\]
This proves the result.
\end{proof}

\subsection{Proof of Proposition~\ref{prop:equicorrelated-coupling}}
\label{app:proof-equicorrelated}

We prove Proposition~\ref{prop:equicorrelated-coupling} from the main paper.

\begin{proof}
First consider the existence statement. If the stated equicorrelated Gaussian
coupling exists, then its block covariance matrix has the form
\[
    \Sigma_c = R_c\otimes I_d,
\]
where
\[
    R_c
    =
    \begin{pmatrix}
    1 & c & \cdots & c\\
    c & 1 & \cdots & c\\
    \vdots & \vdots & \ddots & \vdots\\
    c & c & \cdots & 1
    \end{pmatrix}
    =
    (1-c)I_K + c\,\mathbf 1_K\mathbf 1_K^\top .
\]
A centered Gaussian vector with covariance \(\Sigma_c\) exists if and only if
\(\Sigma_c\) is positive semidefinite. Since
\(\Sigma_c=R_c\otimes I_d\), this is equivalent to \(R_c\) being positive
semidefinite.

The eigenvalues of \(R_c\) are explicit. For any
\(v\in\mathbb R^K\) satisfying \(\mathbf 1_K^\top v=0\),
\[
    R_c v = (1-c)v .
\]
Hence \(1-c\) is an eigenvalue with multiplicity \(K-1\). Also,
\[
    R_c\mathbf 1_K
    =
    \bigl(1+(K-1)c\bigr)\mathbf 1_K .
\]
Thus the remaining eigenvalue is \(1+(K-1)c\). Therefore \(R_c\) is positive
semidefinite if and only if
\[
    1-c\ge 0
    \qquad\text{and}\qquad
    1+(K-1)c\ge 0.
\]
Equivalently,
\[
    -\frac{1}{K-1}\le c\le 1 .
\]
This proves the existence part.

We  prove the minimax property. Let
\(\gamma\in\Gamma_K\) be any coupling of \(K\) standard Gaussian vectors
\(Z_1,\ldots,Z_K\in\mathbb R^d\). Since each marginal is \(\bN(0,I_d)\),
\[
    \mathbb E\|Z_i\|_2^2=d,
    \qquad i=1,\ldots,K .
\]
By nonnegativity of squared norm,
\[
    0
    \le
    \mathbb E\left\|\sum_{i=1}^K Z_i\right\|_2^2 .
\]
Expanding the right-hand side gives
\begin{align*}
    0
    &\le
    \sum_{i=1}^K \mathbb E\|Z_i\|_2^2
    +
    2\sum_{1\le i<j\le K}
    \mathbb E\langle Z_i,Z_j\rangle
    \\
    &=
    Kd
    +
    2\sum_{1\le i<j\le K}
    \mathbb E\langle Z_i,Z_j\rangle .
\end{align*}
Hence
\[
    \frac{2}{K(K-1)}
    \sum_{1\le i<j\le K}
    \frac{1}{d}\mathbb E\langle Z_i,Z_j\rangle
    \ge
    -\frac{1}{K-1}.
\]
The maximum of a finite set is at least its average, so
\[
    \max_{1\le i<j\le K}
    \frac{1}{d}\mathbb E\langle Z_i,Z_j\rangle
    \ge
    -\frac{1}{K-1}.
\]
Since this lower bound holds for every coupling
\(\gamma\in\Gamma_K\),
\[
    \inf_{\gamma\in\Gamma_K}
    \max_{1\le i<j\le K}
    \frac{1}{d}\mathbb E_\gamma\langle Z_i,Z_j\rangle
    \ge
    -\frac{1}{K-1}.
\]

It remains to show that the lower bound is attainable. By the first part of the
proposition, the equicorrelated Gaussian coupling with
\(c=-1/(K-1)\) exists. For this coupling,
\[
    \frac{1}{d}\mathbb E\langle Z_i,Z_j\rangle
    =
    -\frac{1}{K-1},
    \qquad i\ne j .
\]
Therefore,
\[
    \max_{1\le i<j\le K}
    \frac{1}{d}\mathbb E\langle Z_i,Z_j\rangle
    =
    -\frac{1}{K-1}.
\]
This proves the minimax claim and completes the proof.
\end{proof}
\subsection{Optimality of repulsive coupling}\label{subsec:optimality}
\begin{proposition}[Optimality of repulsive coupling for linear-feature diversity]
\label{prop:repulsive-linear-feature-diversity}
Let \(k\ge 2\), \(d\ge 1\), \(m\ge 1\), and let \(J\in\mathbb R^{m\times d}\) be fixed. Let \(\Gamma_k\) denote the set of all couplings of \(k\) standard Gaussian vectors in \(\mathbb R^d\). For \(\gamma\in\Gamma_k\), write
\[
    (z_1,\ldots,z_k)\sim\gamma,
    \qquad
    y_i = Jz_i\in\mathbb R^m .
\]
Define the average squared feature separation
\[
    \mathcal D_J(\gamma)
    :=
    \frac{2}{k(k-1)}
    \sum_{1\le i<j\le k}
    \mathbb E_\gamma \|y_i-y_j\|_2^2 .
\]
Then
\[
    \mathcal D_J(\gamma)
    \le
    \frac{2k}{k-1}\|J\|_F^2
    \qquad
    \text{for every } \gamma\in\Gamma_k .
\]
Moreover, equality holds if and only if
\[
    \sum_{i=1}^k Jz_i = 0
    \qquad
    \gamma\text{-a.s.}
\]
In particular, the repulsive Gaussian coupling
\[
    z_i
    =
    \sqrt{\frac{k}{k-1}}\,(u_i-\bar u),
    \qquad
    \bar u=\frac1k\sum_{j=1}^k u_j,
    \qquad
    u_1,\ldots,u_k\overset{\mathrm{i.i.d.}}{\sim}N(0,I_d),
\]
attains this upper bound.

Now let \(\Gamma_k^{\mathrm G}\subset\Gamma_k\) denote the set of possibly degenerate jointly Gaussian couplings with standard Gaussian marginals. For \(\tau>0\), define the average RBF feature similarity
\[
    \mathcal S_{J,\tau}(\gamma)
    :=
    \frac{2}{k(k-1)}
    \sum_{1\le i<j\le k}
    \mathbb E_\gamma
    \exp\left(
        -\frac{\|y_i-y_j\|_2^2}{2\tau^2}
    \right).
\]
Then, for every \(\gamma\in\Gamma_k^{\mathrm G}\),
\[
    \mathcal S_{J,\tau}(\gamma)
    \ge
    \det\left(
        I_m
        +
        \frac{2k}{(k-1)\tau^2}JJ^\top
    \right)^{-1/2}.
\]
The repulsive Gaussian coupling attains equality. Consequently, among all jointly Gaussian couplings with standard Gaussian marginals, the repulsive Gaussian coupling minimizes the expected average RBF similarity in any fixed linear feature space \(z\mapsto Jz\). The same minimizer also applies to any nonnegative weighted sum, or nonnegative weighted integral, of such RBF similarities over bandwidths \(\tau\).
\end{proposition}

\begin{proof}
We first prove the statement for squared feature separation. For deterministic vectors
\(y_1,\ldots,y_k\in\mathbb R^m\), the following identity holds:
\[
    \sum_{1\le i<j\le k}\|y_i-y_j\|_2^2
    =
    k\sum_{i=1}^k\|y_i\|_2^2
    -
    \left\|\sum_{i=1}^k y_i\right\|_2^2 .
\]
Applying this identity with \(y_i=Jz_i\) and taking expectation under \(\gamma\), we get
\[
    \sum_{1\le i<j\le k}
    \mathbb E_\gamma\|Jz_i-Jz_j\|_2^2
    =
    k\sum_{i=1}^k \mathbb E_\gamma\|Jz_i\|_2^2
    -
    \mathbb E_\gamma
    \left\|
        \sum_{i=1}^k Jz_i
    \right\|_2^2 .
\]
Since each marginal satisfies \(z_i\sim N(0,I_d)\),
\[
    \mathbb E_\gamma\|Jz_i\|_2^2
    =
    \operatorname{tr}(JJ^\top)
    =
    \|J\|_F^2 .
\]
Therefore
\[
    \mathcal D_J(\gamma)
    =
    \frac{2k}{k-1}\|J\|_F^2
    -
    \frac{2}{k(k-1)}
    \mathbb E_\gamma
    \left\|
        \sum_{i=1}^k Jz_i
    \right\|_2^2 .
\]
The second term is nonnegative, so
\[
    \mathcal D_J(\gamma)
    \le
    \frac{2k}{k-1}\|J\|_F^2 .
\]
Equality holds if and only if
\[
    \mathbb E_\gamma
    \left\|
        \sum_{i=1}^k Jz_i
    \right\|_2^2
    =0,
\]
which is equivalent to
\[
    \sum_{i=1}^k Jz_i=0
    \qquad
    \gamma\text{-a.s.}
\]
For the repulsive Gaussian coupling,
\[
    \sum_{i=1}^k z_i
    =
    \sqrt{\frac{k}{k-1}}
    \sum_{i=1}^k (u_i-\bar u)
    =
    0
    \qquad
    \text{a.s.}
\]
Hence \(\sum_i Jz_i=0\) almost surely, and the repulsive Gaussian coupling attains the upper bound.

We now prove the RBF similarity claim. Fix \(\gamma\in\Gamma_k^{\mathrm G}\). For each pair \(i<j\), define
\[
    V_{ij}
    :=
    \operatorname{Cov}(y_i-y_j)
    =
    \operatorname{Cov}(Jz_i-Jz_j).
\]
Since \(\gamma\) is jointly Gaussian, \(y_i-y_j\) is a centered Gaussian vector in \(\mathbb R^m\) with covariance \(V_{ij}\). Hence
\[
    \mathbb E_\gamma
    \exp\left(
        -\frac{\|y_i-y_j\|_2^2}{2\tau^2}
    \right)
    =
    \det\left(I_m+\frac{1}{\tau^2}V_{ij}\right)^{-1/2}.
\]
Define
\[
    g_\tau(V)
    :=
    \det\left(I_m+\frac{1}{\tau^2}V\right)^{-1/2}
\]
for positive semidefinite \(V\in\mathbb R^{m\times m}\). The function \(g_\tau\) is convex on the positive semidefinite cone. Indeed,
\[
    V\mapsto
    -\frac12\log\det\left(I_m+\frac{1}{\tau^2}V\right)
\]
is convex, and the exponential function is convex and nondecreasing. Also, \(g_\tau\) is decreasing in the positive semidefinite order: if \(0\preceq A\preceq B\), then
\[
    \det\left(I_m+\frac{1}{\tau^2}A\right)
    \le
    \det\left(I_m+\frac{1}{\tau^2}B\right),
\]
and therefore \(g_\tau(A)\ge g_\tau(B)\).

Let
\[
    \bar V
    :=
    \frac{2}{k(k-1)}
    \sum_{1\le i<j\le k} V_{ij}.
\]
Using the matrix form of the pairwise-difference identity,
\[
    \sum_{1\le i<j\le k}
    (y_i-y_j)(y_i-y_j)^\top
    =
    k\sum_{i=1}^k y_i y_i^\top
    -
    \left(\sum_{i=1}^k y_i\right)
    \left(\sum_{i=1}^k y_i\right)^\top .
\]
Taking expectations gives
\[
    \sum_{1\le i<j\le k} V_{ij}
    =
    k\sum_{i=1}^k \operatorname{Cov}(y_i)
    -
    \operatorname{Cov}\left(\sum_{i=1}^k y_i\right).
\]
Since \(\operatorname{Cov}(y_i)=JJ^\top\) for every \(i\),
\[
    \bar V
    =
    \frac{2k}{k-1}JJ^\top
    -
    \frac{2}{k(k-1)}
    \operatorname{Cov}\left(\sum_{i=1}^k y_i\right).
\]
Thus
\[
    \bar V
    \preceq
    V_\star,
    \qquad
    V_\star
    :=
    \frac{2k}{k-1}JJ^\top .
\]
By Jensen's inequality and the convexity of \(g_\tau\),
\[
    \mathcal S_{J,\tau}(\gamma)
    =
    \frac{2}{k(k-1)}
    \sum_{1\le i<j\le k} g_\tau(V_{ij})
    \ge
    g_\tau(\bar V).
\]
Since \(g_\tau\) is decreasing in the positive semidefinite order and \(\bar V\preceq V_\star\),
\[
    g_\tau(\bar V)
    \ge
    g_\tau(V_\star).
\]
Combining the two inequalities yields
\[
    \mathcal S_{J,\tau}(\gamma)
    \ge
    g_\tau(V_\star)
    =
    \det\left(
        I_m
        +
        \frac{2k}{(k-1)\tau^2}JJ^\top
    \right)^{-1/2}.
\]

It remains to check that the repulsive Gaussian coupling attains equality. Under this coupling,
\[
    \operatorname{Cov}(z_i,z_j)
    =
    -\frac{1}{k-1}I_d
    \qquad
    \text{for } i\ne j.
\]
Therefore, for every pair \(i<j\),
\[
    \operatorname{Cov}(z_i-z_j)
    =
    I_d+I_d
    -
    \operatorname{Cov}(z_i,z_j)
    -
    \operatorname{Cov}(z_j,z_i)
    =
    \frac{2k}{k-1}I_d.
\]
Hence
\[
    V_{ij}
    =
    \operatorname{Cov}(Jz_i-Jz_j)
    =
    \frac{2k}{k-1}JJ^\top
    =
    V_\star
\]
for every \(i<j\). Thus
\[
    \mathcal S_{J,\tau}(\gamma)
    =
    g_\tau(V_\star)
    =
    \det\left(
        I_m
        +
        \frac{2k}{(k-1)\tau^2}JJ^\top
    \right)^{-1/2}.
\]
So the lower bound is tight, and it is attained by the repulsive Gaussian coupling.

Finally, since the same repulsive Gaussian coupling attains the lower bound for every bandwidth \(\tau>0\), any nonnegative weighted sum or nonnegative weighted integral over such RBF similarities is minimized by the same coupling. This completes the proof.
\end{proof}

\subsection{Proof of Proposition~\ref{prop:marginal-invariance}}
\begin{proof}
Let \(Z\sim \bN(0,I_d)\) be an independent standard Gaussian noise. Since
\(Z_i\sim \bN(0,I_d)\) for each \(i\), we have
\[
    Z_i \stackrel{d}{=} Z .
\]
Because \(G\) is deterministic, the map \(z\mapsto G(p,z)\) is fixed once
the prompt \(p\) is fixed. Hence applying this same map to two random
variables with the same distribution gives random variables with the same
distribution:
\[
    X_i = G(p,Z_i) \stackrel{d}{=} G(p,Z).
\]
Thus each \(X_i\) has the same marginal distribution as the output generated
from an independent standard Gaussian initial noise.

Now let \(q\) be any integrable single-image score. Since
\(X_i \stackrel{d}{=} G(p,Z)\), we have
\[
    \bE\bigl[q(p,X_i)\bigr]
    =
    \bE_{Z\sim \bN(0,I_d)}\bigl[q(p,G(p,Z))\bigr],
    \qquad i=1,\ldots,K .
\]
By linearity of expectation,
\[
\begin{aligned}
    \bE\left[K^{-1}\sum_{i=1}^K q(p,X_i)\right]
    &=
    K^{-1}\sum_{i=1}^K \bE\bigl[q(p,X_i)\bigr] \\
    &=
    K^{-1}\sum_{i=1}^K
    \bE_{Z\sim \bN(0,I_d)}\bigl[q(p,G(p,Z))\bigr] \\
    &=
    \bE_{Z\sim \bN(0,I_d)}\bigl[q(p,G(p,Z))\bigr].
\end{aligned}
\]
This proves the claim.
\end{proof}

\subsection{Proof of Theorem~\ref{thm:coupling parametrization}}
\label{app:proof-matrix-parametrization}

We prove Theorem~\ref{thm:coupling parametrization} from the main paper.

\begin{proof}
We first prove the forward direction. Let
\[
    u_1,\ldots,u_r \overset{\mathrm{i.i.d.}}{\sim} \bN(0,I_d),
\]
and let \(A\in\mathbb R^{M\times r}\) have unit-norm rows:
\[
    \sum_{\ell=1}^r A_{m\ell}^2=1,
    \qquad m=1,\ldots,M .
\]
Define
\[
    z_m=\sum_{\ell=1}^r A_{m\ell}u_\ell,
    \qquad m=1,\ldots,M .
\]
Since each \(z_m\) is a linear combination of independent Gaussian random
vectors, \((z_1,\ldots,z_M)\) is jointly Gaussian. Also,
\[
    \mathbb E[z_m]=0 .
\]
For the marginal covariance,
\begin{align*}
    \operatorname{Cov}(z_m)
    &=
    \operatorname{Cov}
    \left(
    \sum_{\ell=1}^r A_{m\ell}u_\ell
    \right)
    \\
    &=
    \sum_{\ell=1}^r A_{m\ell}^2
    \operatorname{Cov}(u_\ell)
    \\
    &=
    \left(\sum_{\ell=1}^r A_{m\ell}^2\right)I_d
    =
    I_d .
\end{align*}
Hence
\[
    z_m\sim \bN(0,I_d),
    \qquad m=1,\ldots,M .
\]
Thus, the construction gives a valid coupling of \(M\) standard Gaussian noise
vectors.

For the cross-covariance, using the independence of
\(u_1,\ldots,u_r\),
\begin{align*}
    \operatorname{Cov}(z_i,z_j)
    &=
    \operatorname{Cov}
    \left(
    \sum_{\ell=1}^r A_{i\ell}u_\ell,
    \sum_{s=1}^r A_{js}u_s
    \right)
    \\
    &=
    \sum_{\ell=1}^r A_{i\ell}A_{j\ell} I_d
    \\
    &=
    \left(\sum_{\ell=1}^r A_{i\ell}A_{j\ell}\right)I_d .
\end{align*}
Equivalently,
\[
    \operatorname{Cov}(z_i,z_j)
    =
    (AA^\top)_{ij}I_d .
\]

We now prove the converse. Suppose
\((z_1,\ldots,z_M)\) is a centered jointly Gaussian coupling such that
each \(z_m\sim\bN(0,I_d)\) and
\[
    \operatorname{Cov}(z_i,z_j)=\rho_{ij}I_d .
\]
Let \(R=(\rho_{ij})_{i,j=1}^M\). Since \(R\) is the sample-level covariance
matrix of a jointly Gaussian vector, \(R\) is positive semidefinite. Also,
because each marginal covariance is \(I_d\),
\[
    R_{ii}=\rho_{ii}=1,
    \qquad i=1,\ldots,M .
\]

Assume \(\operatorname{rank}(R)\le r\). Let
\(s=\operatorname{rank}(R)\). Since \(R\) is positive semidefinite, there exists
\(B\in\mathbb R^{M\times s}\) such that
\[
    R=BB^\top .
\]

If \(s<r\), pad \(B\) with \(r-s\) zero columns and define
\[
    A=[B\ \ 0]\in\mathbb R^{M\times r}.
\]
If \(s=r\), set \(A=B\). Then
\[
    AA^\top=R .
\]
Moreover, for every row,
\[
    \|A_{i,:}\|_2^2
    =
    (AA^\top)_{ii}
    =
    R_{ii}
    =
    1 .
\]
Thus \(A\) has unit-norm rows.

Now, construct new Gaussian noise vectors by drawing
\[
    u_1,\ldots,u_r \overset{\mathrm{i.i.d.}}{\sim} \bN(0,I_d)
\]
and setting
\[
    \tilde z_m=\sum_{\ell=1}^r A_{m\ell}u_\ell .
\]
By the forward direction,
\[
    \operatorname{Cov}(\tilde z_i,\tilde z_j)
    =
    (AA^\top)_{ij}I_d
    =
    R_{ij}I_d
    =
    \rho_{ij}I_d .
\]
The constructed vector
\((\tilde z_1,\ldots,\tilde z_M)\) is centered Gaussian and has the same block
covariance matrix as \((z_1,\ldots,z_M)\). Since centered Gaussian laws are
determined by their covariance matrices,
\[
    (\tilde z_1,\ldots,\tilde z_M)
    \overset{d}{=}
    (z_1,\ldots,z_M).
\]
This proves the converse and completes the proof.
\end{proof}

\subsection{First-order effect of a coupling}\begin{theorem}[First-order effect of a Gaussian coupling]
\label{thm:coupling-taylor}
Fix a prompt \(p\), a gallery size \(K\), and a smooth gallery-level objective
\[
    H_p:\mathbb R^{Kd}\to \mathbb R,
    \qquad
    H_p(z_1,\ldots,z_K).
\]
Let \(R\in\mathbb R^{K\times K}\) be a positive semidefinite correlation
matrix with \(R_{ii}=1\), and let
\[
    B := R-I_K .
\]
For \(t\in[0,1]\), define
\[
    R_t := I_K+tB = (1-t)I_K+tR,
\]
and let
\[
    (Z_1^{(t)},\ldots,Z_K^{(t)})
    \sim \bN(0,R_t\otimes I_d).
\]
Define
\[
    \Psi(t)
    :=
    \bE\left[
        H_p(Z_1^{(t)},\ldots,Z_K^{(t)})
    \right].
\]
For \(1\le i<j\le K\), write
\[
    \mathcal D_{ij}H_p(z_1,\ldots,z_K)
    :=
    \operatorname{tr}\!\left(\nabla^2_{ij}H_p(z_1,\ldots,z_K)\right)
    =
    \sum_{\ell=1}^d
    \frac{\partial^2 H_p}{\partial z_{i,\ell}\partial z_{j,\ell}} .
\]
Assume the derivatives below are integrable. Then
\[
    \Psi(1)-\Psi(0)
    =
    \int_0^1
    \sum_{1\le i<j\le K}
    B_{ij}\,
    \bE_t\!\left[
        \mathcal D_{ij}H_p(Z_1^{(t)},\ldots,Z_K^{(t)})
    \right]dt ,
\]
where \(\bE_t\) denotes expectation under the Gaussian law with covariance
\(R_t\otimes I_d\). Consequently, the first-order expansion around the
independent coupling is
\[
\begin{aligned}
    \bE_R[H_p(Z_1,\ldots,Z_K)]
    &=
    \bE_{\mathrm{iid}}[H_p(Z_1,\ldots,Z_K)]  \\
    &\quad+
    \sum_{1\le i<j\le K}
    B_{ij}\,
    \bE_{\mathrm{iid}}\!\left[
        \mathcal D_{ij}H_p(Z_1,\ldots,Z_K)
    \right]
    +
    \operatorname{Rem}(B),
\end{aligned}
\]
where
\[
\begin{aligned}
    \operatorname{Rem}(B)
    &=
    \int_0^1 (1-t)
    \sum_{1\le i<j\le K}
    \sum_{1\le k<\ell\le K}
    B_{ij}B_{k\ell}                                      \\
    &\qquad\qquad\qquad\cdot
    \bE_t\!\left[
        \mathcal D_{k\ell}\mathcal D_{ij}
        H_p(Z_1^{(t)},\ldots,Z_K^{(t)})
    \right]dt .
\end{aligned}
\]
In particular, if
\[
    \left|
    \bE_t\!\left[
        \mathcal D_{k\ell}\mathcal D_{ij}H_p
        (Z_1^{(t)},\ldots,Z_K^{(t)})
    \right]
    \right|
    \le M
\]
for all \(t\in[0,1]\) and all index pairs, then
\[
    |\operatorname{Rem}(B)|
    \le
    \frac{M}{2}
    \left(
        \sum_{1\le i<j\le K}|B_{ij}|
    \right)^2 .
\]
\end{theorem}

\begin{proof}
Let
\[
    \Sigma_t := R_t\otimes I_d .
\]
Since \(R_t=(1-t)I_K+tR\), the matrix \(R_t\) is positive semidefinite for all
\(t\in[0,1]\). Hence \(\Sigma_t\) is a valid covariance matrix.

We use the standard Gaussian differentiation identity: if
\(X_t\sim \bN(0,\Sigma_t)\) and \(f\) is smooth enough, then
\[
    \frac{d}{dt}\bE[f(X_t)]
    =
    \frac{1}{2}
    \sum_{a,b}
    \frac{d(\Sigma_t)_{ab}}{dt}
    \bE\left[
        \frac{\partial^2 f(X_t)}{\partial x_a\partial x_b}
    \right].
\]
Apply this identity to
\[
    f(z_1,\ldots,z_K)=H_p(z_1,\ldots,z_K).
\]
For coordinates \(z_{i,r}\) and \(z_{j,s}\),
\[
    \frac{d}{dt}
    \operatorname{Cov}(Z_{i,r}^{(t)},Z_{j,s}^{(t)})
    =
    B_{ij}\mathbf 1\{r=s\}.
\]
The diagonal terms vanish because \(B_{ii}=0\). For every pair \(i<j\), the
off-diagonal covariance appears twice in the full covariance matrix, once as
\((i,j)\) and once as \((j,i)\). These two terms cancel the factor \(1/2\).
Therefore
\[
    \Psi'(t)
    =
    \sum_{1\le i<j\le K}
    B_{ij}
    \bE_t\!\left[
        \sum_{r=1}^d
        \frac{\partial^2 H_p}
        {\partial z_{i,r}\partial z_{j,r}}
        (Z_1^{(t)},\ldots,Z_K^{(t)})
    \right].
\]
That is,
\[
    \Psi'(t)
    =
    \sum_{1\le i<j\le K}
    B_{ij}
    \bE_t\!\left[
        \mathcal D_{ij}H_p(Z_1^{(t)},\ldots,Z_K^{(t)})
    \right].
\]
Integrating from \(t=0\) to \(t=1\) gives the exact interpolation identity.

For the Taylor expansion, write
\[
    \Psi(1)=\Psi(0)+\Psi'(0)+
    \int_0^1(1-t)\Psi''(t)\,dt .
\]
The expression for \(\Psi'(0)\) gives the first-order term under independent
noise. Applying the same Gaussian differentiation identity once more to
\(\mathcal D_{ij}H_p\) gives
\[
    \Psi''(t)
    =
    \sum_{1\le i<j\le K}
    \sum_{1\le k<\ell\le K}
    B_{ij}B_{k\ell}
    \bE_t\!\left[
        \mathcal D_{k\ell}\mathcal D_{ij}
        H_p(Z_1^{(t)},\ldots,Z_K^{(t)})
    \right].
\]
Substituting this into Taylor's formula gives the stated remainder. The bound
follows by taking absolute values and using the assumed uniform bound \(M\).
\end{proof}

\section{Details and additional experiments}
All experiments were conducted on eight NVIDIA L40S GPUs.
Section~\ref{subsub:configs-diverse} provides the experimental configurations.
  Section~\ref{subsub:learned-optimal} describes how amortized coupling optimization is applied in this setting.

\subsection{Diversity generation} \label{app:subsec:diverse}
\subsubsection{Experiments configs}\label{subsub:configs-diverse}
 We evaluate diverse text-to-image generation on Stable Diffusion 1.5 (SD1.5), Stable Diffusion XL
  (SDXL), and Stable Diffusion 3 (SD3). For each prompt, we generate a gallery of $K=3$ images and compare four
  methods: i.i.d. Gaussian initial noise, i.i.d. noise followed by Contrastive Noise Optimization \cite{kim2025diverse} (CNO),
  repulsively coupled Gaussian noise, and repulsively coupled Gaussian noise followed by CNO. The coupled methods
  preserve the standard Gaussian marginal distribution of each initial noise and only change the dependence structure
  within the gallery.

  For CNO, we follow the model-specific hyperparameters of the original CNO paper, summarized in
  Table~\ref{tab:cno_hparams}. Unlike the original CNO evaluation, which uses a batch size $5$ for
  SD1.5 and SD3, and batch size $3$ for SDXL, we use a fixed gallery size $K=3$ for all models to keep the comparison
  with coupled-noise initialization uniform.

  We evaluate diversity using average pairwise L2 distance, MSS, and Vendi score. L2 is computed directly in pixel
  space as the mean pairwise per-pixel squared distance. MSS and Vendi are computed from SSCD image
  features \cite{pizzi2022self}: MSS is the mean off-diagonal cosine similarity within each gallery, while Vendi is computed from the
  eigenvalue entropy of the SSCD feature similarity matrix.

  \begin{table}[t]
      \centering
      \caption{CNO and sampling hyperparameters for diverse generation.}
      \label{tab:cno_hparams}
      \begin{tabular}{lccccc}
      \toprule
      Model & CFG scale & CNO steps & Learning rate & Pooling window & Sampling steps \\
      \midrule
      SD1.5 & 6.0 & 3 & 0.01 & 16 & 50 \\
      SDXL  & 6.0 & 3 & 0.01 & 16 & 50 \\
      SD3   & 7.0 & 3 & 0.001 & 32 & 28 \\
      \bottomrule
      \end{tabular}
  \end{table}

\subsubsection{Compute Cost Comparison}

  We compare the generation cost of the four sampling settings in Table~\ref{tab:diverse_compute_cost}.
  The reported time is the average wall-clock time to generate one $K=3$ gallery for a single prompt. Since repulsive
  coupling only applies a linear transformation to the initial Gaussian noises, its cost is effectively the same as
  i.i.d. sampling. CNO introduces an optimization overhead because it optimizes the noises before the final diffusion sampling pass; in our setting, this overhead is modest since only 3 optimization steps are used.

  \begin{table}[t]
      \centering
      \caption{Compute cost for diverse generation. Time is measured per prompt for one $K=3$ gallery. Relative time
  is normalized by i.i.d. sampling for each model.}
      \label{tab:diverse_compute_cost}
      \begin{tabular}{llcc}
      \toprule
      Model & Method & Time / prompt & Relative time \\
      \midrule
      SD1.5 & i.i.d. & 8.42s & $1.00\times$ \\
            & i.i.d. + CNO & 8.92s & $1.06\times$ \\
            & coupled & 8.42s & $1.00\times$ \\
            & coupled + CNO & 8.92s & $1.06\times$ \\
      \midrule
      SDXL  & i.i.d. & 18.27s & $1.00\times$ \\
            & i.i.d. + CNO & 19.06s & $1.04\times$ \\
            & coupled & 18.27s & $1.00\times$ \\
            & coupled + CNO & 19.06s & $1.04\times$ \\
      \midrule
      SD3   & i.i.d. & 64.73s & $1.00\times$ \\
            & i.i.d. + CNO & 71.88s & $1.11\times$ \\
            & coupled & 64.73s & $1.00\times$ \\
            & coupled + CNO & 71.88s & $1.11\times$ \\
      \bottomrule
      \end{tabular}
  \end{table}

 \subsubsection{Ablation study on the Gallary Size}\label{subsub:ablation-k-diversity}
  We additionally study how the effect of repulsive coupling changes with the gallery size $K$. We run an SD1.5
  ablation on 500 COCO prompts and compare i.i.d. Gaussian noise with repulsively coupled Gaussian noise for
  $K \in \{2,3,4,5\}$. 
  Results are shown in Table~\ref{tab:k_ablation_diversity}. Repulsive coupling improves all three diversity metrics
  for every value of $K$: it increases pixel-space L2, decreases MSS, and increases Vendi. The L2 gain is largest for
  $K=2$ and decreases as $K$ grows, from a $49.7\%$ relative improvement at $K=2$ to a $13.0\%$ improvement at $K=5$.
 Even for larger galleries, the
  coupled
  initialization continues to improve SSCD-based diversity, as reflected by lower MSS and higher Vendi across all
  settings.

  \begin{table}[t]
      \centering
      \caption{Ablation on gallery size for SD1.5 diverse generation over 500 COCO prompts. We compare i.i.d.
  Gaussian
      noise and repulsively coupled Gaussian noise without CNO.}
      \label{tab:k_ablation_diversity}
      \begin{tabular}{llccc}
      \toprule
      $K$ & Method & L2 $\uparrow$ & MSS $\downarrow$ & Vendi $\uparrow$ \\
      \midrule
      2 & i.i.d. & 0.1276 & 0.1550 & 1.9679 \\
        & coupled & \textbf{0.1911} & \textbf{0.1414} & \textbf{1.9721} \\
      \midrule
      3 & i.i.d. & 0.1267 & 0.1504 & 2.9136 \\
        & coupled & \textbf{0.1567} & \textbf{0.1370} & \textbf{2.9255} \\
      \midrule
      4 & i.i.d. & 0.1260 & 0.1502 & 3.8372 \\
        & coupled & \textbf{0.1458} & \textbf{0.1426} & \textbf{3.8519} \\
      \midrule
      5 & i.i.d. & 0.1257 & 0.1497 & 4.7419 \\
        & coupled & \textbf{0.1421} & \textbf{0.1439} & \textbf{4.7592} \\
      \bottomrule
      \end{tabular}
  \end{table}
 \subsubsection{Ablation study on the correlation strength}\label{subsub:ablation-c-diversity}

  We additionally study how the strength of the equicorrelated coupling affects diversity for a fixed gallery size.
  We run an SD1.5 ablation on 100 COCO prompts with gallery size \(K=3\), and vary the off-diagonal correlation
  \(c\) in the equicorrelated Gaussian coupling from independent sampling, \(c=0\), to the maximally repulsive valid
  value, \(c=-1/(K-1)=-1/2\). We evaluate
  \[
  c \in \{0, -0.125, -0.25, -0.375, -0.5\}.
  \]
  Results are shown in Table~\ref{tab:c_ablation_diversity}. Increasing the amount of negative correlation
  consistently
  improves pixel-space L2, from \(0.1283\) at \(c=0\) to \(0.1574\) at \(c=-0.5\). SSCD-based diversity also
  improves:
  MSS decreases and Vendi increases for all negative-correlation settings compared with independent sampling. The
  best
  MSS and Vendi occur at \(c=-0.375\), while the $c = -0.5$ result is extremely close.
  The results show that stronger repulsion generally improves diversity, with the theoretical boundary \(c=-1/(K-1)\)
  performing best or near-best across the reported metrics.

  \begin{table}[t]
      \centering
      \caption{Ablation on equicorrelation strength for SD1.5 diverse generation over 100 COCO prompts with
      gallery size \(K=3\). We vary the off-diagonal correlation \(c\) in the Gaussian coupling.}
      \label{tab:c_ablation_diversity}
      \begin{tabular}{lccc}
      \toprule
      Correlation \(c\) & L2 \(\uparrow\) & MSS \(\downarrow\) & Vendi \(\uparrow\) \\
      \midrule
      \(0\)        & 0.1283 & 0.1542 & 2.9066 \\
      \(-0.125\)  & 0.1353 & 0.1429 & 2.9176 \\
      \(-0.25\)   & 0.1428 & 0.1447 & 2.9204 \\
      \(-0.375\)  & 0.1497 & \textbf{0.1332} & \textbf{2.9300} \\
      \(-0.5\)    & \textbf{0.1574} & 0.1346 & 2.9281 \\
      \bottomrule
      \end{tabular}
  \end{table}
\subsubsection{Learning the optimal coupling}\label{subsub:learned-optimal}
We also experimented with learned couplings optimized directly for image-space diversity.
  For this experiment, we parameterize a batch coupling matrix \(A\in\mathbb{R}^{K\times K}\) following \ref{subsec:optimize}, sample independent
  basis noises \(\epsilon_1,\ldots,\epsilon_K\), and form coupled initial noises
  \[
      z_i=\sum_{j=1}^K A_{ij}\epsilon_j .
  \]
  Each row of \(A\) is constrained to have unit norm, so every \(z_i\) remains marginally standard Gaussian. We
  initialize \(A\) either as the identity coupling, and update it by
  projected gradient steps on the row-normalized sphere.

  The optimization uses Stable Diffusion 1.5 and treats the final generated images as the forward signal. Since
  differentiating through the full denoising trajectory is expensive, we use a skip-gradient estimator: we first run
  the deterministic sampler from the coupled initial noises to obtain final latents \(z_i^{\mathrm{final}}\), then
  use
  \[
      \tilde z_i
      =
      \mathrm{stopgrad}(z_i^{\mathrm{final}})
      +
      z_i
      -
      \mathrm{stopgrad}(z_i)
  \]
  before VAE decoding. We tested gallery sizes \(K=4\) and \(K=5\) using pairwise SSIM and LPIPS as the optimization objective. 
  In all cases, the learned coupling quickly approaches the analytic repulsive structure. For example, for
  \(K=4\), the best learned sample-level correlation matrix using SSIM is
  \[
  C=A A^\top =
  \begin{pmatrix}
  1.000 & -0.339 & -0.332 & -0.333 \\
  -0.339 & 1.000 & -0.335 & -0.335 \\
  -0.332 & -0.335 & 1.000 & -0.326 \\
  -0.333 & -0.335 & -0.326 & 1.000
  \end{pmatrix}.
  \]
  This is very close to the repulsive coupling, whose off-diagonal entries are
  \(-1/(K-1)=-1/3\). This provides an empirical check that, when the learning objective directly rewards image-space
  diversity, amortized optimization recovers essentially the same repulsive dependence structure as the analytic
  construction.

\subsection{Brightness-Clustered Generation}
  \label{subsec:brightness}
Noise coupling can also be learned to induce structured, user-specified relationships among generated samples.
We consider a brightness-clustered generation task in which a four-image gallery is encouraged to
  split into two brighter and two darker samples.
  We follow the amortized optimization framework \eqref{eq:coupling-level-opt} and learn a coupling matrix \(A\) that maps independent Gaussian basis noises \(\epsilon_1,\ldots,
  \epsilon_4\) into coupled initial noises $z_i=\sum_{j=1}^4 A_{ij}\epsilon_j$. Given
  generated images \(x_i=G(p,z_i)\), let \(b_i\) denote the average brightness of image \(i\),
  computed from the value channel. We optimize \(A\) using the objective
  \begin{align}\label{eqn:brightness-objective}
  \left|
  \frac{b_1+b_2}{2}
  -
  \frac{b_3+b_4}{2}
  \right|
  +
  \lambda
  \frac{|b_1-b_2|+|b_3-b_4|}{2},
  \end{align}
  where the first term encourages a bright--dark split between the two image pairs, and the second
  term weakly encourages non-collapse within each pair. After learning \(A\), generation is
  performed without any per-prompt optimization: we sample independent basis noises, apply the
  learned coupling, and run the pretrained diffusion sampler normally.

In our experiment, we choose $\lambda = 0.35$. The coupling matrix \(A\) is learned on 600 COCO prompts using SD1.5 and then applied
  amortized, without per-prompt optimization, to held-out prompts across SD1.5, SDXL, and SD3. We
  compare independent noise, repulsive equicorrelated noise with correlation \(-1/3\), and the
  learned coupling.

  We report three metrics: the 
  pair gap $ \left|(b_1 + b_2)/2 - (b_3+b_4)/2\right|$
   the contrast objective from Eq.~\eqref{eqn:brightness-objective}, and  sorted margin $
  b_{(3)} - b_{(2)},$
  where \(b_{(1)}\le b_{(2)}\le b_{(3)}\le b_{(4)}\) are the sorted brightness values. This metric
  measures the gap between the darker and brighter halves of the gallery.

\begin{table}
    \centering
    \caption{Brightness-clustered generation.}
    \label{tab:brightness_cluster}
    \begin{tabular}{llccc}
    \toprule
    Model & Method & Gap $\uparrow$ & Contrast $\uparrow$ & Margin $\uparrow$ \\
    \midrule
    \rowcolor{blue!8}
    SD1.5 & i.i.d. & 0.0448 & 0.0673 & 0.0336 \\
    \rowcolor{blue!8}
          & equicorr. & 0.0494 & 0.0746 & 0.0384 \\
    \rowcolor{blue!8}
          & learned & \textbf{0.0702} & \textbf{0.0873} & \textbf{0.0462} \\
    \midrule
    \rowcolor{green!8}
    SDXL & i.i.d. & 0.0355 & 0.0524 & 0.0256 \\
    \rowcolor{green!8}
         & equicorr. & 0.0364 & 0.0543 & 0.0271 \\
    \rowcolor{green!8}
         & learned & \textbf{0.0461} & \textbf{0.0604} & \textbf{0.0286} \\
    \midrule
    \rowcolor{orange!10}
    SD3 & i.i.d. & 0.0478 & 0.0715 & 0.0353 \\
    \rowcolor{orange!10}
        & equicorr. & 0.0508 & 0.0761 & 0.0380 \\
    \rowcolor{orange!10}
        & learned & \textbf{0.0652} & \textbf{0.0845} & \textbf{0.0412} \\
    \bottomrule
    \end{tabular}
  \end{table}
The learned brightness coupling produces a clear block structure in the noise correlations. The two samples
  intended to form the bright cluster are positively correlated, and the two samples intended to form the dark
  cluster are also positively correlated, while samples across the two clusters are strongly negatively correlated.
  In one learned solution, the induced correlation matrix $C =AA^\top$ is
  \[
  C =
  \begin{pmatrix}
  1.00 & 0.47 & -0.76 & -0.74 \\
  0.47 & 1.00 & -0.75 & -0.77 \\
  -0.76 & -0.75 & 1.00 & 0.55 \\
  -0.74 & -0.77 & 0.55 & 1.00
  \end{pmatrix}.
  \]
  This structure matches the target relationship: samples (1,2) share noise components and therefore tend to move
  together in brightness, samples (3,4) do the same, and the two groups are coupled in opposite directions to
  encourage a bright--dark split across the gallery.

 Table~\ref{tab:brightness_cluster} shows that the learned coupling gives a cleaner brightness
  split than both independent and repulsive noise across all models. This suggests that coupling
  optimization can learn user-desired gallery structure, rather than only increasing generic
  diversity. Moreover, although the coupling is learned using SD1.5, it generalizes to other model
  families such as SDXL and SD3. Once learned, the coupling can be reused amortized without additional cost. Figure \ref{fig:brightness_cluster_examples} provides visualizations.

\subsection{Background generation with fixed object}\label{app:subsec:fixed-object}
\subsubsection{Experiment Details}\label{subsec:objective exp details}
  We evaluate fixed-object background generation on 100 LSUN bedroom images, using YOLOv8 segmentation masks for the
  bed foreground. Failed detections or generations are excluded, leaving 92 valid images for the Stable Diffusion 1.5 experiment. All metrics are computed with the
  same foreground mask used for generation. Boundary-TV is computed in a 5-pixel dilation-minus-erosion band around
  the mask boundary.

We use \texttt{runwayml/stable-diffusion-v1-5} with prompt ``a bedroom'',
  DDIM inversion/sampling with 20 steps, and classifier-free guidance scale 1. We generate $K=8$ samples per image.
  Our method uses DDIM-inverted foreground noise, repulsively coupled background noise, Adam optimization with
  learning rate $0.01$, and checkpointed optimization steps at 10, 20, 30, 50, and 100 iterations. The optimization
  uses the same masked foreground MSE objective and resets the background latent after each Adam step. SDEdit is
  implemented with a DDIM scheduler, 20 denoising steps, guidance scale 1, strengths $\{0.2,0.4,0.6,0.8\}$, and
  foreground pasting before evaluation. In our experiments, strengths of 0.2, 0.4, and 0.6 produced very low diversity, so the table reports results at strength 0.8. Stable Diffusion inpainting uses 20 steps,
  guidance scale 1, strength 1, and the inverted foreground mask as the editable region.

\subsubsection{Effect of optimization steps}

Table~\ref{tab:fixed_object_steps} reports how the number of optimization steps affects each evaluation axis for SD1.5 fixed-object background generation. The results show a controllable trade-off between background diversity and foreground fidelity, governed by a single hyperparameter.

As the number of optimization steps increases from 10 to 100, foreground preservation improves monotonically: RMSE-F decreases from $0.086$ to $0.048$, indicating that additional optimization progressively pulls the masked region closer to the input object. At the same time, background diversity decreases: LPIPS-B drops from $0.453$ to $0.373$ and L2-B from $0.045$ to $0.029$. This is consistent with the structure of the objective in \eqref{eqn:fixed-object-opt}, which only constrains the foreground; longer optimization tightens this constraint and indirectly reduces variation in the surrounding pixels through the coupling between foreground and background in the denoising trajectory.

Boundary-TV remains essentially flat across all settings $0.012$--$0.013$, confirming that the smooth object--background transition is a property of the coupling-based initialization rather than of the optimization budget. Global naturalness, measured by BRISQUE, degrades mildly with more steps (from $17.52$ to $23.02$). 

Across all optimization budgets, our method achieves substantially lower BRISQUE and Boundary-TV than both SDEdit and SD inpaint, indicating that the coupled-noise approach consistently produces more natural images that blend smoothly with the preserved object. With sufficient optimization steps, foreground fidelity also surpasses that of SD inpaint. SDEdit retains the strongest foreground preservation by construction, but it is dominated on every other axis we measure (background diversity, boundary smoothness, and global naturalness).

\begin{table}[htbp!]
      \centering
      \caption{Effect of optimization steps for SD1.5 fixed-object background generation. We report mean (standard
  deviation). LPIPS-B and L2-B measure background diversity; RMSE-F measures foreground preservation; Boundary-TV
  measures local roughness near the object boundary; BRISQUE measures global image naturalness.}
      \label{tab:fixed_object_steps}
      \resizebox{0.95\linewidth}{!}{
      \begin{tabular}{lccccc}
      \toprule
      Steps & LPIPS-B $\uparrow$ & L2-B $\uparrow$ & RMSE-F $\downarrow$ & Boundary-TV $\downarrow$ & BRISQUE $
  \downarrow$ \\
      \midrule
      10  & 0.453 (0.160) & 0.045 (0.013) & 0.086 (0.025) & 0.012 (0.004) & 17.52 (9.57) \\
      20  & 0.433 (0.166) & 0.040 (0.014) & 0.072 (0.021) & 0.012 (0.004) & 18.12 (10.84) \\
      30  & 0.418 (0.171) & 0.037 (0.014) & 0.065 (0.019) & 0.012 (0.004) & 19.11 (11.05) \\
      50  & 0.399 (0.174) & 0.033 (0.015) & 0.057 (0.016) & 0.012 (0.004) & 20.24 (10.63) \\
      100 & 0.373 (0.177) & 0.029 (0.016) & 0.048 (0.014) & 0.013 (0.004) & 23.02 (11.15) \\
      \bottomrule
      \end{tabular}}
  \end{table}
\subsubsection{Unconditional diffusion models}
We evaluate the unconditional fixed-object task on 100 LSUN bedroom images with YOLOv8 bed masks as foreground
  objects. We use the pretrained pixel-space diffusion model google/ddpm-bedroom-256 and generate a gallery of $K = 4$
  images for each input. Our method uses DDIM inversion to initialize the foreground noise, repulsively coupled
  background noise for diversity, and 100 optimization steps to improve foreground fidelity while keeping the
  background noise fixed. We compare against DDIM inversion, Object-inv. + random BG, and SDEdit. DDIM inversion
  reconstructs from the inverted noise and therefore has no meaningful background diversity. Object-inv. + random BG
  uses the inverted foreground noise with random background noise. 
\begin{table}[t]
      \centering
      \caption{Unconditional fixed-object background generation. We report mean (standard deviation). LPIPS-B and L2-
  B measure background diversity; RMSE-F measures foreground preservation; Boundary-TV measures local roughness near
  the object boundary; BRISQUE measures global image naturalness.}
      \label{tab:uncond_fixed_object}
      \resizebox{0.9\linewidth}{!}{
      \begin{tabular}{lccccc}
      \toprule
      Method & LPIPS-B $\uparrow$ & L2-B $\uparrow$ & RMSE-F $\downarrow$ & Boundary-TV $\downarrow$ & BRISQUE $
  \downarrow$ \\
      \midrule
      DDIM inv. & -- & -- & \textbf{0.20} (0.06) & \textbf{0.015} (0.006) & 32.98 (7.82) \\
      Object-inv. + random BG & 0.39 (0.16) & \underline{0.22} (0.09) & 0.32 (0.08) & \underline{0.017} (0.011) &
  38.37 (11.00) \\
      SDEdit & \underline{0.39} (0.12) & 0.14 (0.04) & \underline{0.03} (0.01) & 0.041 (0.010) & \textbf{24.45}
  (8.35) \\
      Ours & \textbf{0.40} (0.15) & \textbf{0.25} (0.10) & 0.05 (0.02) & 0.039 (0.014) &
  \underline{31.46} (11.91) \\
      \bottomrule
      \end{tabular}}
  \end{table}

Object-inv. + random BG produces diverse backgrounds, but has poor
  foreground preservation, with RMSE-F 0.32. This shows that fixing only the object noise is insufficient after
  denoising. SDEdit-80 achieves the best foreground preservation and BRISQUE, with RMSE-F 0.03 and BRISQUE 24.45, but
  has lower background diversity than ours, especially in L2-B, 0.14 versus 0.25. Our method achieves the strongest
  background diversity, with LPIPS-B 0.40 and L2-B 0.25, while keeping foreground RMSE much lower than Object-inv. +
  random BG, 0.05 versus 0.32.

\section{Visualizations}

  \begin{figure}[htbp!]
    \centering
    \includegraphics[width=\textwidth]{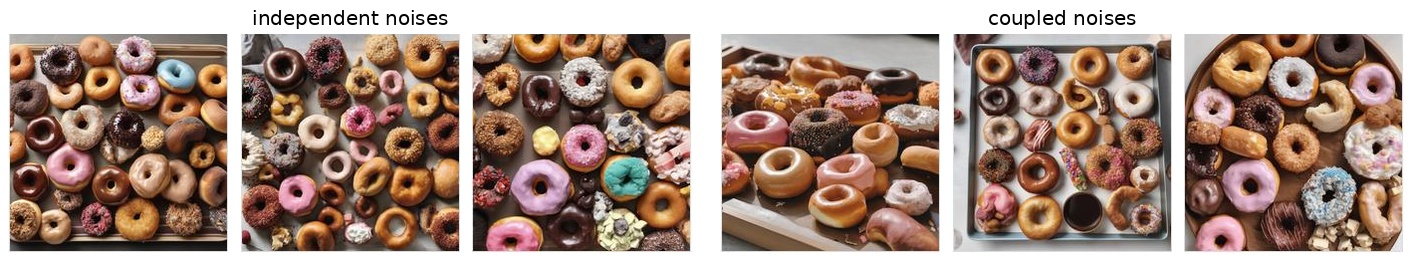}
    \caption{Independent versus coupled noises with the same prompt ``Assortment of doughnuts and other snack items
  on a serving tray.'' generated from SDXL.}
      \label{fig:diverse_img_sdxl}
  \end{figure}

 \begin{figure}
      \centering
      \includegraphics[width=0.95\linewidth]{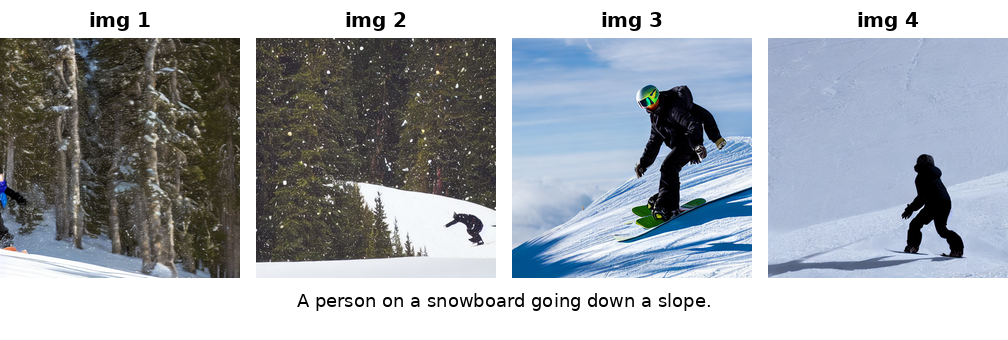}
      \vspace{0.5em}
      \includegraphics[width=0.95\linewidth]{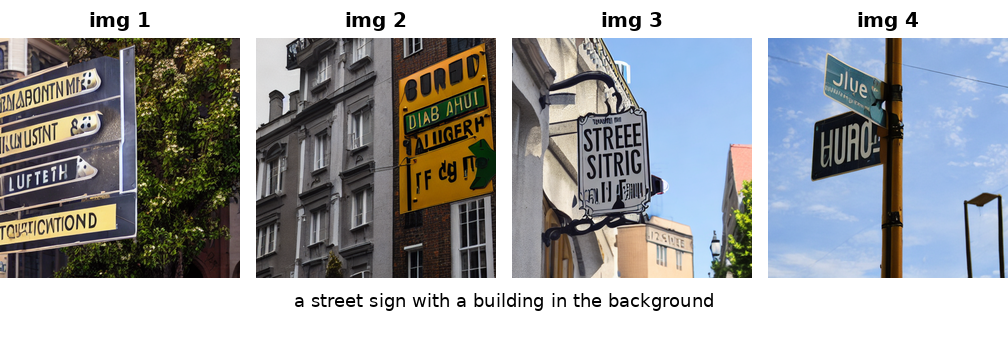}
      \caption{Examples of brightness-clustered generation with Stable Diffusion 1.5. The learned coupling produces
  four-image galleries with a clear bright (img 1, 2) and dark (img 3,4) split while preserving image quality and prompt alignment.}
      \label{fig:brightness_cluster_examples}
  \end{figure}

   \begin{figure}[htbp!]
      \centering
      \includegraphics[width=0.95\linewidth]{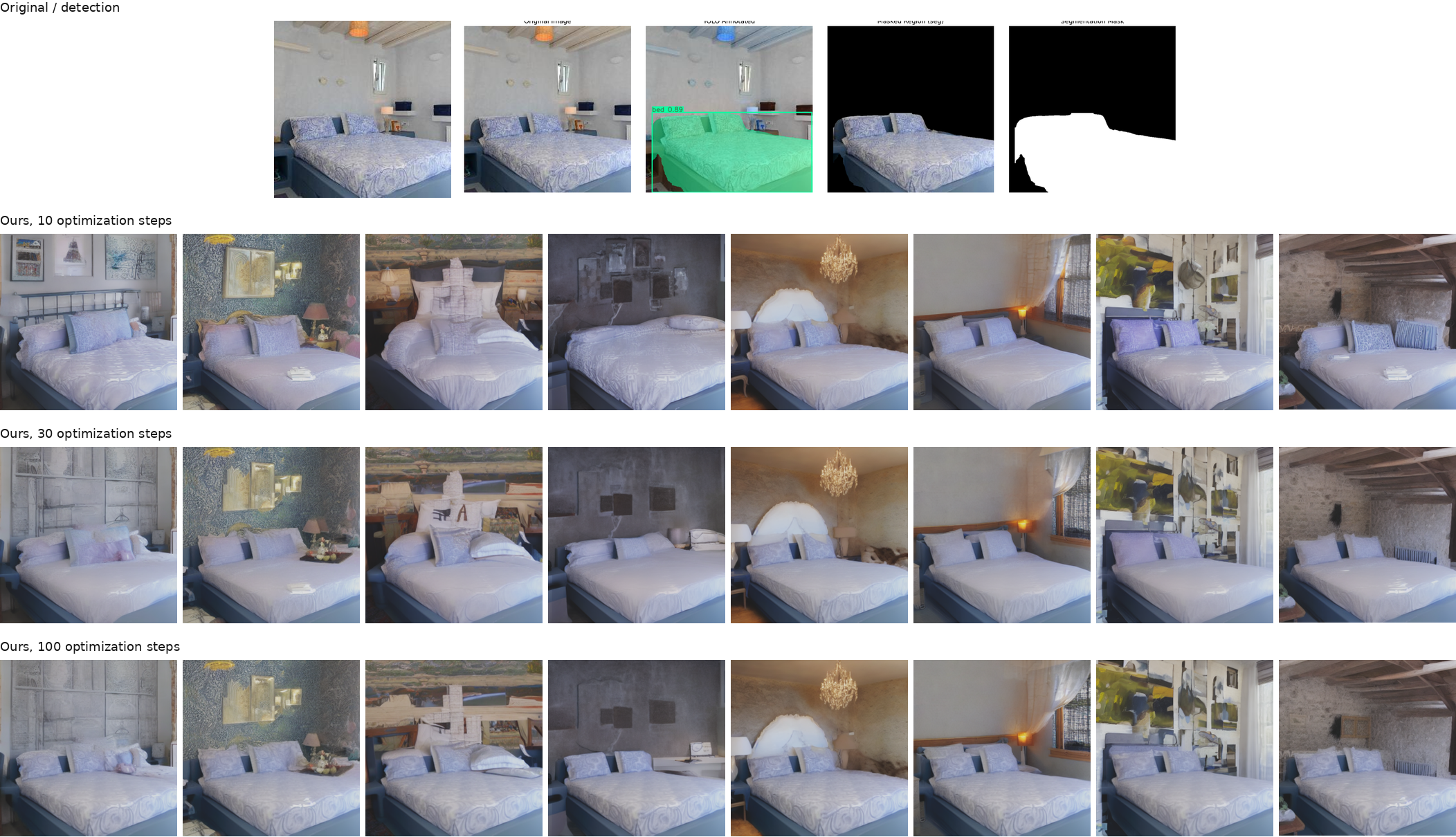}
      \caption{Fixed-object background generation on an LSUN bedroom image. The first row shows the input image, YOLOv8
  bed detection, masked foreground, and segmentation mask. The remaining rows show our method after 10, 30, and 100
  optimization steps. Across optimization budgets, the bed is preserved while the surrounding room layout, wall
  texture, lighting, and decor vary substantially. Increasing the number of optimization steps improves foreground
  fidelity while maintaining diverse and natural backgrounds.}
      \label{fig:bedroom-object}
  \end{figure}
\end{document}